%% file: main.tex
\documentclass[11pt,a4paper]{article}
\input{preamble}

\title{From Sequential Nodes to GPU Batches: Parallel Branch and Bound for Optimal $k$-Sparse GLMs}

\author[$\ast,\dagger$]{Jiachang Liu}
\author[$\ddagger$]{Andrea Lodi}
\affil[$\ast$]{Center for Data Science for Enterprise and Society, Cornell University, Ithaca, USA}
\affil[$\dagger$]{\mbox{School of Operations Research and Information Engineering, Cornell University, Ithaca, USA}}
\affil[$\ddagger$]{Jacobs Technion-Cornell Institute, Cornell Tech and Technion--IIT, New York, USA\\
\texttt{\{jiachang.liu, andrea.lodi\}@cornell.edu}}

\date{}

\begin{document}
\maketitle
\addtocontents{toc}{\protect\setcounter{tocdepth}{-10}}

\input{sections/0_Abstract.tex}

\input{sections/1_Introduction.tex}

\input{sections/2_Related_Work.tex}
\input{sections/3_Preliminaries.tex}
\input{sections/4_Methodology.tex}
\input{sections/5_Experiments.tex}
\input{sections/6_Conclusion.tex}

\section*{Acknowledgements}
This work used the Delta system at the National Center for Supercomputing Applications through allocation CIS250029 from the Advanced Cyberinfrastructure Coordination Ecosystem: Services \& Support (ACCESS) program, which is supported by National Science Foundation grants \#2138259, \#2138286, \#2138307, \#2137603, and \#2138296.

\bibliographystyle{plainnat}
\bibliography{references}

\input{sections/7_Appendix.tex}

\end{document}

%% file: preamble.tex
\usepackage[margin=1in]{geometry}
\usepackage[utf8]{inputenc}
\usepackage[T1]{fontenc}
\usepackage{authblk}
\usepackage{graphicx}
\usepackage{tikz}
\usepackage{enumitem}
\usepackage[ruled,vlined,linesnumbered]{algorithm2e}
\usepackage[numbers,compress]{natbib}
\usepackage{microtype}
\usepackage{xcolor}
\usepackage[normalem]{ulem}
\usepackage{nicefrac}
\input{math_macros.tex}

%% file: math_macros.tex
\usepackage{amsthm}
\usepackage{amsmath}
\allowdisplaybreaks
\usepackage{bbm}

\usepackage{amssymb}

\usepackage{bm}
\usepackage{hyperref}
\usepackage{url}
\usepackage{booktabs} 
\usepackage{multirow}

\usepackage{rotating}
\usepackage{tcolorbox}

\newtheorem{theorem}{Theorem}[section]
\newtheorem*{namedtheorem}{Theorem}

\newcommand{\bc}{\bm{c}}

\newcommand{\bq}{\bm{q}}

\newcommand{\bu}{\bm{u}}

\newcommand{\bv}{\bm{v}}
\newcommand{\bx}{\bm{x}}

\newcommand{\bX}{\bm{X}}
\newcommand{\bQ}{\bm{Q}}

\newcommand{\bz}{\bm{z}}

\newcommand{\bs}{\bm{s}}

\newcommand{\by}{\bm{y}}

\newcommand{\balpha}{\bm{\alpha}}

\newcommand{\bB}{\bm{B}}
\newcommand{\bK}{\bm{K}}
\newcommand{\bG}{\bm{G}}
\newcommand{\bR}{\bm{R}}
\newcommand{\bS}{\bm{S}}

\newcommand{\bU}{\bm{U}}

\newcommand{\bbeta}{\bm{\beta}}

\newcommand{\brho}{\bm{\rho}}

\newcommand{\bpi}{\bm{\pi}}

\newcommand{\bPhi}{\bm{\Phi}}
\newcommand{\bSigma}{\bm{\Sigma}}

\newcommand{\bZeta}{\bm{\mathcal{Z}}}
\newcommand{\bzeta}{\bm{\zeta}}

\newcommand{\bbR}{\mathbb{R}}

\newcommand{\ind}{\mathbbm{1}}

\newcommand{\calF}{\mathcal{F}}
\newcommand{\calG}{\mathcal{G}}

\newcommand{\calJ}{\mathcal{J}}

\newcommand{\calL}{\mathcal{L}}

\newcommand{\calN}{\mathcal{N}}

\newcommand{\calR}{\mathcal{R}}
\newcommand{\calS}{\mathcal{S}}
\newcommand{\calT}{\mathcal{T}}
\newcommand{\calV}{\mathcal{V}}

\DeclareMathOperator*{\argmax}{arg\,max}
\DeclareMathOperator*{\argmin}{arg\,min}

\def\1{\bm{1}}

\def\vo{{\bm{o}}}

\DeclareMathAlphabet{\mathsfit}{\encodingdefault}{\sfdefault}{m}{sl}
\SetMathAlphabet{\mathsfit}{bold}{\encodingdefault}{\sfdefault}{bx}{n}

\def\vert#1{\lvert #1 \rvert}

\definecolor{predcolor}{gray}{0.95}
\definecolor{scorecolor}{gray}{0.95}
\definecolor{riskcolor}{gray}{0.95}
\definecolor{transparentcolor}{gray}{0.95}

\definecolor{ForestGreen}{rgb}{0.13, 0.55, 0.13} % http://latexcolor.com/
\usepackage{color} % Load color package                                                                     
\DeclareMathOperator*{\TopSum}{TopSum}
\DeclareMathOperator{\prox}{prox}
\undef\st
\DeclareMathOperator{\st}{s.\!t.\!}

\usetikzlibrary{arrows.meta}

%% file: sections/0_Abstract.tex
\begin{abstract}
GPUs have significantly accelerated first-order methods for large-scale optimization, especially in continuous optimization.
However, this success has not transferred cleanly to problems with discrete variables, combinatorial structure, and nonlinear objectives, such as certifying optimal solutions for cardinality-constrained generalized linear models.
Major challenges include the sequential processing of heterogeneous nodes in branch and bound (BnB) and frequent data movement between the CPU and GPU.
We propose a simple, generic, and modular CPU--GPU framework that processes multiple BnB nodes in batches on GPUs.
The framework is built around a small set of GPU-efficient routines and uses padding together with lightweight custom kernels to handle irregular node data structures.
Experiments show one to two orders of magnitude speedups and zero optimality gap on challenging instances.
The framework can also be extended to collect the entire Rashomon set, enabling downstream statistical analysis such as variable-importance analysis and model selection under secondary user-specific measures (\textit{e.g.}, AUC in classification).
\end{abstract}

%% file: sections/1_Introduction.tex
\section{Introduction}\label{sec:introduction}
GPUs have become a central computing platform for large-scale optimization in machine learning.
Early and successful applications include neural networks~\citep{raina2009large,krizhevsky2012imagenet}.
More recently, GPUs have also been used to scale large linear, quadratic, and conic optimization problems, especially through first-order methods~\citep{applegate2021practical,lu2023cupdlp,lu2023practical,lin2025practicalgpu}.
For example, PDLP~\citep{applegate2021practical} can solve linear programs with millions of variables.
This progress is possible because the main cost in many first-order methods is gradient computation, which reduces largely to matrix-vector operations that GPUs can execute efficiently.

However, this success has not transferred cleanly to (machine-learning)
problems involving discrete and combinatorial structures.
In this paper, we aim to solve the following cardinality-constrained generalized linear models (GLMs) at scale:
\begin{align}
    \min_{\bbeta\in\bbR^p}
    \left\{
        f(\bX\bbeta,\by) + \lambda_2\|\bbeta\|_2^2
        \ :\
        \|\bbeta\|_0\le k,\ \|\bbeta\|_\infty\le M
    \right\}.
    \label{eq:intro_sparse_glm}
\end{align}
Here, $\bX\in\bbR^{n\times p}$ is the feature matrix, $\by$ is the response, $f$ is a convex differentiable GLM loss, $k$ is the sparsity budget, $M$ is a coefficient bound, and $\lambda_2>0$ is the ridge coefficient.
To see the connection to the discrete and combinatorial optimization more explicitly, we introduce binary support indicators $\bz\in\{0,1\}^p$ and rewrite~\eqref{eq:intro_sparse_glm} in the mixed-integer nonlinear programming (MINLP) formulation
\begin{align}
    \min_{\bbeta\in\bbR^p,\bz\in\{0,1\}^p}
    \left\{
        f(\bX\bbeta,\by)+\lambda_2\|\bbeta\|_2^2
        :
        \bm 1^\top\bz\le k,\ |\beta_j|\le Mz_j,\ j\in[p]
    \right\}.
    \label{eq:intro_sparse_glm_mip}
\end{align}
Solving exact cardinality-constrained problems is important in scientific, medical, financial, and operational settings~\citep{ustun2016supersparse,ustun2019learning,liu2022fastsparse,liu2022fasterrisk,liu2024okridge}, especially when we want high predictive performance using only a small set of variables.
In high-dimensional settings with highly correlated features, approximation-based methods (\textit{e.g.}, lasso) can produce poor solutions.
Although~\eqref{eq:intro_sparse_glm_mip} focuses on a specific model class, it captures many key MINLP elements.
Progress on this problem class can therefore inform GPU-accelerated optimization beyond sparse GLMs.

However, two major barriers prevent GPUs from delivering similar speedups for exact sparse GLM certification.
The first barrier is that standard branch and bound (BnB) processes the search tree sequentially at the node level.
To certify optimality, BnB repeatedly partitions the feasible region by fixing selected support indicators $z_j$ to either $0$ or $1$.
A short visual primer on this procedure is given in Appendix~\ref{appendix_sec:bnb_primer}.
At each node $\calN$, it computes a valid lower bound, compares this bound with the incumbent objective value (the objective value of the best feasible solution found so far), and then either prunes the node or branches on another variable.
Recent work~\citep{liu2025scalable,liu2026gpufriendlylinearlyconvergentfirstorder} accelerates the lower-bound computation for a single node on GPUs, but the overall tree search still advances one node at a time.
When certification requires exploring millions of nodes, this sequential node processing remains a major bottleneck.

The second barrier is that several important BnB procedures are still usually designed as CPU-side routines.
These include feasible-solution search for improving the incumbent and variable selection for branching.
This issue becomes more pronounced when the lower-bound relaxation is solved only in the coefficient space $\bbeta$, because the fractional relaxed indicators $\bz$ are then not directly available for support selection or branching.
If lower-bound computation runs on the GPU but feasible-solution search and branching remain on the CPU, each node requires repeated CPU--GPU synchronization.
These transfers interrupt the GPU workload and limit the benefit of accelerating the lower-bound solve alone.
Moreover, transferring data between the CPU and GPU for every single node is inefficient.
Modern data-center GPUs rely on High Bandwidth Memory (HBM) to deliver very high on-device memory throughput, so the preferred pattern is to transfer larger batches less frequently and keep repeated numerical work on the device.

In this work, we address these barriers by processing many open BnB nodes together on the GPU.
Our main contributions are:
\begin{enumerate}[label=$\diamond$]
    \item \emph{Hybrid CPU--GPU Framework:}
    We propose a simple and modular CPU--GPU framework for exact BnB on cardinality-constrained GLMs.
    The CPU manages the irregular tree-search logic, including the open-node queue, incumbent updates, and child-node generation.
    The GPU performs the batched numerical work, including lower-bound solves, rounding, re-optimization, and branching-variable selection.
    
    \item \emph{GPU Routines and Padding Strategy:}
    We show that multi-node BnB computation can be organized around a small set of GPU-efficient routines, including matrix--matrix multiplication, columnwise sorting, and gather-and-reduce operations.
    To handle node-specific irregular data structures, we use padding to make the batch representation more uniform.
    This allows optimized GPU routines to handle the expensive uniform work, while custom kernels are reserved for lightweight irregular steps.
    
    \item \emph{Support/Variable Selections on GPUs:}
    We prove that the relaxed indicator variables can be recovered exactly from the relaxed coefficient vector, without solving an additional optimization problem.
    This result justifies selecting feasible supports and branching variables directly from the coefficient values, using batched GPU operations rather than sequential CPU-side routines.
    
    \item \emph{Empirical Performance:}
    Experiments on challenging sparse GLM instances show substantial runtime reductions.
    The batched GPU framework achieves up to one to two orders of magnitude speedup over the one-node-at-a-time GPU baseline and certifies optimality on difficult cases where the baseline leaves a nontrivial optimality gap.
    
    \item \emph{Applications to Rashomon-Set Collection:}
    The same framework can be modified to collect the Rashomon set of near-optimal $k$-sparse GLMs, extending exact Rashomon-set collection beyond sparse decision trees.
    This allows users to compare competing sparse models, study variable importance, and select a model by considering the optimized training objective together with secondary criteria such as accuracy, AUC, calibration, or other task-specific measures.
    
\end{enumerate}

%% file: sections/2_Related_Work.tex
\section{Related Work}\label{sec:related-work}

\paragraph{Mixed-integer Programming for Sparse GLMs}
Mixed-integer programming (MIP) has been widely used to model and solve sparse GLM-type problems.
From an application point of view, MIP-based sparse models have been used to construct scoring systems~\citep{ustun2016supersparse,ustun2019learning,liu2022fasterrisk} and to identify nonlinear dynamical systems~\citep{bertsimas2023learning,liu2024okridge}.
From a theoretical point of view, a substantial line of work studies stronger convex relaxations, perspective reformulations, and convex-hull descriptions for sparse and indicator-variable formulations~\citep{gunluk2010perspective,atamturk2020safe,atamturk2021sparse,wei2020convexification,atamturk2020supermodularity,wei2022ideal,shafiee2024constrained}.
From a computational point of view, many papers develop algorithms for solving sparse regression and classification problems at larger scales~\citep{xie2020scalable,bertsimas2020sparse1,bertsimas2020sparse2,hazimeh2020fast,dedieu2021learning,hazimeh2022sparse,guyard2024el0ps,liu2025scalable,liu2026gpufriendlylinearlyconvergentfirstorder}.
However, the computational core of most existing approaches remains CPU-centered: they either call commercial MIP solvers, design problem-specific branch-and-bound routines, or combine continuous relaxations with CPU-side search and heuristics.
In contrast, our work shifts the computational design toward GPU-parallel branch and bound for the sparse GLM family.

\paragraph{GPU-accelerated Optimization}
GPU acceleration has recently gained traction in continuous optimization, including large-scale linear programming via primal-dual hybrid gradient methods~\citep{applegate2021practical,lu2023cupdlp}, convex quadratic programming~\citep{lu2023practical}, conic programming~\citep{lin2025practicalgpu}, and semidefinite programming~\citep{han2024accelerating}.
For discrete optimization, the use of GPUs is less mature.
Existing work includes GPU-accelerated relaxations~\citep{de2024power,liu2025scalable,liu2026gpufriendlylinearlyconvergentfirstorder}, GPU-accelerated primal heuristics~\citep{corduk2025gpuacceleratedprimalheuristics}, and batched first-order LP methods for MIP subroutines such as strong branching and bound tightening~\citep{blin2026batchedlp}.
The closest concurrent work is~\citet{meng2026gpu}, which develops a GPU-accelerated BnB method specifically for sparse linear regression.
In contrast, we target the broader class of sparse GLMs and use a deliberately modular design: lower-bound computation, rounding, re-optimization, branching-variable selection, batching, scheduling, and Rashomon-pool storage are separate components that can be improved independently.

\paragraph{Rashomon Sets}
Modern research on Rashomon sets, namely collections of near-optimal solutions, studies both algorithmic and statistical questions~\citep{breiman2001statistical}.
On the algorithmic side, TreeFARMS explores the whole Rashomon set of sparse decision trees~\citep{xin2022exploring}, and subsequent work uses Rashomon sets for variable-importance distributions, interactive model editing, active learning, and predictive-equivalence analysis~\citep{donnelly2023rashomon,donnelly2025rashomonproto,babbar2025nearoptimal,mctavish2025predictiveequivalence,rudin2024amazing,nguyen2026realitrees,haghighat2026resolving}.
For sparse GLMs, FasterRisk generates many accurate sparse risk scores with different supports~\citep{liu2022fasterrisk}, and related work studies sets of good generalized additive models with the same support but different coefficients~\citep{zhong2023sparsegam}.
Our work focuses on collecting the entire Rashomon set at the support level.
To the best of our knowledge, our framework is the first to do so for a broad class of statistical models beyond sparse decision trees.

%% file: sections/3_Preliminaries.tex
\section{Preliminaries}
\label{sec:preliminaries}

\paragraph{Lower Bound Computation}
To find a lower bound for each node $\calN$, one effective approach is to perform perspective relaxation~\citep{ceria1999convex, frangioni2006perspective, gunluk2010perspective} by replacing $\|\bbeta\|_2^2$ with $\sum_{j=1}^p \beta_j^2 / z_j$ and relaxing $z_j \in \{0, 1\}$ to $z_j \in [0, 1]$.
We use the usual perspective convention: $\beta_j^2/z_j=0$ when $(\beta_j,z_j)=(0,0)$ and $+\infty$ when $z_j=0$ but $\beta_j\ne0$.
This relaxation is useful because it gives a strong convex lower-bound problem at each node.

Following \citet{liu2026gpufriendlylinearlyconvergentfirstorder}, we can rewrite such a node relaxation as a convex composite problem:
\begin{align}
    \min_{\bbeta\in\bbR^p}
    \Phi_{\calN}(\bbeta)
    :=
    F(\bX\bbeta)+G_{\calN}(\bbeta),
    \label{eq:node_composite}
\end{align}
where $F(\bX\bbeta):=f(\bX\bbeta,\by)$ is the smooth GLM loss,
$G_{\calN}(\bbeta)=2\lambda_2 g_{\calN}(\bbeta)$, and $g_{\calN}(\bbeta)$ is an implicit function defined as
\begin{align}
    g_{\calN}(\bbeta)
    :=
    \inf_{\bz}
    \left\{
    \frac12\sum_{j=1}^p\beta_j^2/z_j:
    \begin{array}{l}
        z_j \in [0, 1]\ \forall j\in\calJ_{f}(\calN),\quad
        \bm{1}^\top\bz\le k,\quad
        |\beta_j|\le M z_j,\\
        z_j=0\ \forall j\in\calJ_0(\calN),\qquad
        z_j=1\ \forall j\in\calJ_1(\calN)
    \end{array}
    \right\},
\end{align}
$\calJ_0(\calN) := \{j:z_j=0\}$ and $\calJ_1(\calN) := \{j:z_j=1\}$ are partial fixing decisions on $\bz$, and $\calJ_f(\calN) := \{j:z_j\in[0, 1]\}$ denotes the set of free coordinates.

Problem~\eqref{eq:node_composite} can be solved efficiently by applying the proximal gradient method~\footnote{For simplicity of presentation, we display the vanilla proximal-gradient update. Acceleration, restart, line search, and adaptive stepsize rules can be incorporated in practice; the key operations for our later batched GPU extension are the gradient calculation and the proximal evaluation.}:
\begin{align}
    \bbeta
    &\leftarrow
    \operatorname{prox}_{\eta G_{\calN}}
    \!\left(\bbeta-\eta \nabla_{\bbeta} F(\bX\bbeta)\right)
    =
    \bu-\rho^{-1}\operatorname{prox}_{\rho g_{\calN}^*}(\rho\bu),
    \label{eq:pgd_step_vector_form}
\end{align}
where $\bu:=\bbeta-\eta \nabla_{\bbeta}F(\bX\bbeta)$ and $\rho:=(2\eta\lambda_2)^{-1}$, and the second equality above follows from the Moreau's identity formula~\citep[Theorem~6.45]{beck2017first}.
The proximal operator of $g_{\calN}$ can be efficiently and exactly evaluated by using sorting and the PAVA algorithm~\citep{busing2022monotone}.

\paragraph{Safe Lower Bound and Pruning}
We need to derive a safe lower bound to prune nodes in BnB.
To do this, first note that the Fenchel dual of problem~\eqref{eq:node_composite} is
\begin{align}
    \Psi_{\calN}(\bzeta)
    =
    -F^*(-\bzeta)-G_{\calN}^*(\bX^\top\bzeta).
    \label{eq:prelim_fenchel_dual}
\end{align}
Since $G_{\calN}=2\lambda_2 g_{\calN}$, by using the conjugate scaling rule, we have
\begin{align}
    G_{\calN}^*(\bu)
    &=
    2\lambda_2
    g_{\calN}^*
    \left(
        \frac{\bu}{2\lambda_2}
    \right),\\
    \text{and} \qquad g_{\calN}^*(\bq)
    &=
    \sum_{j\in\calJ_1(\calN)}H_M(q_j)
    +
    \TopSum_{\bar k(\calN)}
    \{H_M(q_j)\}_{j\in\calJ_f(\calN)},
    \label{eq:prelim_node_conjugate}
\end{align}
where $\bar k(\calN)=k-|\calJ_1(\calN)|$, $\TopSum_{\bar k(\calN)}$ sums the largest $\bar k(\calN)$ entries, and $H_M(\cdot)$ is the Huber loss with the threshold parameter $M$.
At each iteration $t$, after obtaining a primal iterate $\bbeta^t$ from~\eqref{eq:pgd_step_vector_form}, we form a dual iterate
\begin{align}
    \bzeta^t=-\nabla F(\bX\bbeta^t).
    \label{eq:dual_candidate_vector_form}
\end{align}
By weak duality~\citep{rockafellar1970convex}, we have $\Psi_{\calN}(\bzeta)\le \Phi_{\calN}(\bbeta)$ for any $\bbeta$ and $\bzeta$.
Therefore, $\Psi_{\calN}(\bzeta^t) \leq \min_{\bbeta} \Phi_{\calN} (\bbeta)$ for any dual iterate $\bzeta^t$, which means that $\Psi_{\calN}(\bzeta^t)$ is a safe lower bound.
Whenever $\Psi_{\calN}(\bzeta^t)$ is greater than or equal to the loss of the incumbent, we prune the node $\calN$.

\paragraph{Feasible Solution and Variable Branching}
Although the composite formulation in~\eqref{eq:node_composite} lets us solve each node relaxation over $\bbeta$, it does not explicitly return the relaxed indicator vector $\bz$.
We still need support candidates for incumbent updates and branching variables for child-node generation.
Existing OKGLM implementations handle these tasks on the CPU, using beam-search heuristics for feasible solutions and deletion-based scores for branching~\citep{liu2025scalable,liu2026gpufriendlylinearlyconvergentfirstorder}.
These CPU-side steps can be time-consuming, are hard to run on GPUs, and require frequent CPU--GPU communication across many BnB nodes.

%% file: sections/4_Methodology.tex
\section{Methodology}\label{sec:methodology}

\subsection{Hybrid CPU--GPU BnB framework}
We propose a hybrid CPU--GPU BnB framework, summarized in Figure~\ref{fig:gpu_bnb_blueprint}, for solving cardinality-constrained GLMs and for guiding future GPU implementations of mixed-integer nonlinear programs.

\input{figures/gpu_bnb_blueprint.tex}

The CPU owns the irregular tree logic: it stores open nodes, schedules batches, reconstructs node constraints and warm starts, updates the incumbent and global lower bound, prunes nodes, branches unresolved nodes, and inserts children back into the queue.
The GPU owns the dense numerical work: it processes many nodes in a batch, solves lower-bound relaxations, evaluates primal and dual bounds, and searches for feasible solutions in parallel.

The framework is intentionally modular.
Node ordering, lower-bound solves, feasible-solution search, and branching are separate components, so each can be replaced or improved without redesigning the full solver.
For example, node ordering may use breadth-first, depth-first, best-bound, or incumbent-guided rules from the BnB literature~\citep{linderoth1999computational,morrison2016branch}; lower-bound and re-optimization routines may use different first-order variants; and branching may use any GPU-friendly score (including scores inspired by classical MIP branching rules~\citep{achterberg2005branching,achterberg2009scip}) that returns a free feature index $j\in\calJ_f(\calN)$.
The main design requirement is that these choices should still expose batched GPU work whenever possible.
This organization keeps the exact BnB certificate on the CPU while turning the repeated numerical subproblems into GPU-efficient batched routines.

\subsection{Parallel multi-node lower-bound computation}
\label{subsec:parallel_multinode_lower_bound}

Let a GPU batch contain $m$ BnB nodes $\calN_1,\ldots,\calN_m$.
We store the coefficient iterates columnwise as
\begin{align}
    \bB=[\bbeta^{(1)},\ldots,\bbeta^{(m)}]\in\bbR^{p\times m},
\end{align}
where $\bbeta^{(b)}$ denote the coefficient vector of node $\calN_b$.
Collectively, we are trying to minimize the sum of $m$ independent relaxation objectives:
\begin{align}
    \Phi_{\mathrm{batch}}(\bB)
    :=
    \sum_{b=1}^m
    \Phi_{\calN_b}(\bbeta^{(b)})
    =
    \underbrace{\sum_{b=1}^m F(\bX\bbeta^{(b)})}_{\mathcal F_{\mathrm{batch}}(\bB)}
    +
    \underbrace{\sum_{b=1}^m G_{\calN_b}(\bbeta^{(b)})}_{\mathcal G_{\mathrm{batch}}(\bB)}.
\end{align}
A matrix form of proximal gradient descent, analogous to~\eqref{eq:pgd_step_vector_form}, is
\begin{align}
    \bB
    \leftarrow
    \operatorname{prox}_{\eta \mathcal G_{\mathrm{batch}}}
    \!\left(
        \bB-\eta\nabla_{\bB}\mathcal F_{\mathrm{batch}}(\bB)
    \right).
    \label{eq:pgd_step_matrix_form}
\end{align}

\paragraph{Calculate gradients with uniform structure}
To calculate $\nabla_{\bB}\mathcal F_{\mathrm{batch}}(\bB)$, it boils down to matrix-matrix operations and element-wise operations on the loss function.
For clarity, we write the loss function $f$ explicitly in the separable form
\[
    f(\bs,\by)=\sum_{i=1}^n \ell(s_i,y_i),
\]
where $\bs=\bX\bbeta$ is the linear predictor and $\ell$ is the scalar loss for one observation.
Now, we can compute $\nabla_{\bB}\mathcal F_{\mathrm{batch}}(\bB)$ as
\begin{align}
    \nabla_{\bB}\mathcal F_{\mathrm{batch}}(\bB)=\bX^\top\bR,\qquad \text{where} \qquad
    R_{i,b}=\partial_s\ell(s,y_i)\big|_{s=S_{i,b}} \quad \text{and} \quad
    \bS=\bX\bB.
    \label{eq:gradient_matrix_form}
\end{align}
Formula~\eqref{eq:gradient_matrix_form} makes the gradient step naturally GPU-friendly.
The products $\bX\bB$ and $\bX^\top\bR$ are matrix--matrix multiplications and can be delegated to optimized general matrix--matrix multiplication (GEMM) routines, while $\bR$ is obtained by applying the scalar derivative $\partial_s\ell$ independently to each entry.
Thus, the gradient computation follows a straightforward SIMD (same instruction, multiple data) pattern, even when the batch contains nodes from different parts of the BnB tree.
The same regularity is not available when evaluating the proximal operator of $\calG_{\text{batch}}$.

\paragraph{Evaluating proximal operators with non-uniform node structure}
The proximal step is column-separable, but it is not uniform across the batch because each node can have different fixed-in, fixed-out, and free coordinates.
With $\bU:=\bB-\eta\nabla_{\bB}\calF_{\mathrm{batch}}(\bB)$ and $\rho:=(2\eta\lambda_2)^{-1}$, we have
\[
    \bB
    \leftarrow
    \bU-\rho^{-1}\operatorname{prox}_{\rho \mathcal G_{\mathrm{batch}}^*}(\rho\bU)
    \quad
    \Longleftrightarrow
    \quad
    \bbeta^{(b)}
    \leftarrow
    \bu^{(b)}-\rho^{-1}\operatorname{prox}_{\rho g_{\calN_b}^*}(\rho\bu^{(b)})
    \; \forall \; b=1,\ldots,m.
\]
The difficulty is that the valid coordinates and reduced cardinality budgets differ by column.
On the one hand, if we run the sorting--PAVA routine of~\citet{liu2026gpufriendlylinearlyconvergentfirstorder} one column at a time, we would underuse the GPU.
On the other hand, if we hand-write a custom kernel for the entire per-column sorting--PAVA routine, we would lose the opportunity to use existing optimized batched sorting routines.
\input{figures/padded_column_sorting_example.tex}

To overcome this problem, we pad each column before sorting.
Coordinates outside $\calJ_f(\calN_b)$ receive sentinel keys such as $-\infty$, so all columns have the same apparent length $p$ and the expensive sort can now use optimized batched GPU routines; Figure~\ref{fig:padded_column_sorting} illustrates this idea.
Only the remaining node-specific work is handled by lightweight custom kernels: passing free counts, reduced budgets, and sorted free magnitudes to PAVA, then scattering the resulting free-coordinate values back and applying the Moreau's identity formula.
This gives the same proximal update as the node-by-node algorithm, but moves the expensive sorting step into a batched GPU operation.
Appendix~\ref{appendix_sec:more_efficient_PAVA} describes a slightly more efficient PAVA algorithm than proposed in~\citet{liu2025scalable} by exploiting the sorted structure.

\subsection{Rounding, re-optimization, variable selection, and branching}

A key limitation of~\citet{liu2026gpufriendlylinearlyconvergentfirstorder} is that feasible solutions are obtained by a CPU-based beam search.
We instead use the relaxed coefficient vector $\bbeta^\star$ itself to perform rounding and choose branching variables for $\bz$.
The justification is that the relaxed indicator vector $\bz^\star$ can be recovered from $\bbeta^\star$ without solving another optimization problem.

\begin{theorem}[Recovering relaxed indicators from relaxed coefficients]
\label{thm:recover_relaxed_indicators}
Fix a BnB node $\calN$ and let $\bbeta^\star$ solve problem~\eqref{eq:node_composite}.
Let $\bar k=k-|\calJ_1(\calN)|$ and $p_f=|\calJ_f(\calN)|$.
Set $z_j^\star=0$ for $j\in\calJ_0(\calN)$ and $z_j^\star=1$ for $j\in\calJ_1(\calN)$.
On the free set $\calJ_f(\calN)$, if $\bar k=0$, set $z_j^\star=0$.
If at most $\bar k$ free coefficients are nonzero, set $z_j^\star=1$ for nonzero $\beta_j^\star$ and $z_j^\star=0$ for zero $\beta_j^\star$.
Otherwise, sort the free magnitudes as
$|\beta_{\pi(1)}^\star|\ge\cdots\ge|\beta_{\pi(p_f)}^\star|$.
Find an index $s\in\{0,\ldots,\bar k-1\}$ such that
\begin{equation}
\label{eq:z_recovery_threshold_step}
    \tau
    :=
    \frac{\sum_{r=s+1}^{p_f}|\beta_{\pi(r)}^\star|}{\bar k-s},
    \qquad
    |\beta_{\pi(s)}^\star|\ge \tau \ge |\beta_{\pi(s+1)}^\star|,
\end{equation}
with the convention $|\beta_{\pi(0)}^\star|=+\infty$.
Then set
\begin{equation}
\label{eq:z_recovery_z_step}
    z_{\pi(r)}^\star
    =
    \begin{cases}
        1, & r\le s,\\
        |\beta_{\pi(r)}^\star|/\tau, & r>s,
    \end{cases}
    \qquad r=1,\ldots,p_f.
\end{equation}
This gives an optimal relaxed indicator vector $\bz^\star$ paired with $\bbeta^\star$.
\end{theorem}

\noindent The proof is given in Appendix~\ref{app:proof_recover_relaxed_indicators}.
In practice, we can only obtain an approximate solution $\hat{\bbeta}$ instead of the exact optimal solution $\bbeta^\star$.
However, we can still apply equations~\eqref{eq:z_recovery_threshold_step} and~\eqref{eq:z_recovery_z_step} to obtain an approximate solution $\hat{\bz}$.

We can actually use the theorem to do rounding and branching without explicitly obtaining $\bz^\star$ or $\hat{\bz}$.
For rounding, we take $\widehat{\calS}_{\calN}=\calJ_1(\calN)\cup\calT_{\bar k}$, where $\calT_{\bar k}$ contains the $\bar k$ largest values of $|\beta_j^\star|$ over $j\in\calJ_f(\calN)$, and fix the corresponding support indicators to $1$ and the remaining free support indicators to $0$.
This support contains the $k$ largest recovered relaxed indicators of $\bz^\star$.
For branching, we choose $j^\star\in\argmax_{j\in\calJ_f(\calN)}|\beta_j^\star|$, which is also the free coordinate with the largest recovered indicator value.
Thus, support restriction and branching can both be implemented directly from $\bbeta^\star$.

After selecting $\widehat{\calS}_{\calN_b}$ by rounding, coefficients are re-optimized over this support with all other coefficients fixed to zero and $\|\widetilde{\bbeta}^{(b)}\|_\infty\le M$.
We can perform this re-optimization using any proximal gradient method, where the proximal step is a projection onto $[-M,M]$.
Let $I_r^{(b)}$ be the $r$th selected feature for node $b$.
The batched gradient and predictor computations are
\[
    \left(\nabla\widetilde{\phi}_b(\widetilde{\bbeta}^{(b)})\right)_r
    =
    \sum_{i=1}^n X_{i,I_r^{(b)}}R_{i,b}
    +
    2\lambda_2\widetilde{\beta}^{(b)}_r,\quad
    R_{i,b}=\partial_s\ell(s,y_i)\big|_{s=S_{i,b}},\quad
    S_{i,b}=\sum_{r=1}^k X_{i,I_r^{(b)}}\widetilde{\beta}^{(b)}_r.
\]
These are gather-and-reduce operations over features, samples, and batch columns, so rounding, re-optimization, and branching can all run on GPUs without returning to computations on the CPU side.

\subsection{Dual solutions, primal objectives, and dual objectives}

After a batched proximal-gradient update produces $\bB=[\bbeta^{(1)},\ldots,\bbeta^{(m)}]$, let $\bZeta=[\bzeta^{(1)},\ldots,\bzeta^{(m)}]\in\bbR^{n\times m}$ be the matrix of batched dual variables, and let $\bQ=[\bq^{(1)},\ldots,\bq^{(m)}]$ be the matrix of batched input to $g_{\calN_b}^*(\cdot)$ in~\eqref{eq:prelim_node_conjugate}.
Motivated by the vector construction in~\eqref{eq:prelim_fenchel_dual},~\eqref{eq:prelim_node_conjugate}, and~\eqref{eq:dual_candidate_vector_form}, we can construct the batched quantities in matrix forms:
\begin{align}
    \bS=\bX\bB,
    \qquad
    R_{i,b}=\partial_s\ell(s,y_i)\big|_{s=S_{i,b}},
    \qquad
    \bZeta=-\bR,
    \qquad
    \bQ=\frac{1}{2\lambda_2}\bX^\top\bZeta.
    \label{eq:batched_dual_candidate_matrix_form}
\end{align}
These calculations enjoy the same uniform structure as the gradient calculation in~\eqref{eq:gradient_matrix_form}: $\bX\bB$ and $\bX^\top\bZeta$ are GEMMs, while $\bR$ and $\bZeta$ are computed by entrywise GPU kernels.

Let $\bPhi(\bB):=\left(\Phi_{\calN_1}(\bbeta^{(1)}),\ldots,\Phi_{\calN_m}(\bbeta^{(m)})\right)^\top$ and $\bm{\Psi}(\bZeta):=\left(\Psi_{\calN_1}(\bzeta^{(1)}),\ldots,\Psi_{\calN_m}(\bzeta^{(m)})\right)^\top$ be the vectors of batched primal and dual objectives, where
\begin{align}
    [\bPhi(\bB)]_b
    &=
    \sum_{i=1}^n \ell(S_{i,b},y_i)
    +
    2\lambda_2 g_{\calN_b}(\bbeta^{(b)}),
    \label{eq:batched_primal_value_component}\\
    [\bm{\Psi}(\bZeta)]_b
    &=
    -\sum_{i=1}^n \ell^*(-\zeta_{i,b},y_i)
    -
    2\lambda_2 g_{\calN_b}^*(\bq^{(b)}).
    \label{eq:batched_dual_value_component}
\end{align}
Both objectives can be evaluated in parallel on GPUs.
The smooth terms $\ell(S_{i,b},y_i)$ and $\ell^*(-\zeta_{i,b},y_i)$ are entrywise loss evaluations followed by column reductions.
The node-dependent terms $g_{\calN_b}(\bbeta^{(b)})$ and $g_{\calN_b}^*(\bq^{(b)})$ are less uniform because each node $\calN_b$ imposes different constraints on $\bz$.
As with the proximal operator evaluation, padding smooths out these irregularities so that we can apply batched sorting.
After sorting, lightweight custom kernels construct the majorization vector needed by Algorithm~1 of~\citet{liu2026gpufriendlylinearlyconvergentfirstorder} to evaluate $g_{\calN_b}(\bbeta^{(b)})$, and separately perform $\TopSum_{\bar k(\calN_b)}$ needed to evaluate $g_{\calN_b}^*(\bq^{(b)})$.

The same batched formulation also extends naturally to multi-GPU settings; Appendix~\ref{appendix_sec:multi_gpu_distributed_computing} describes both node-parallel and row-distributed variants.

\subsection{Rashomon-set collection}

For a threshold $\epsilon\ge0$, the support-level sparse GLM Rashomon set is
\[
    \calR_\epsilon^{\mathrm{supp}}
    =
    \left\{
    S\subseteq[p]:
    |S|\le k,\ 
    v(S)\le (1+\epsilon)\Phi^\star
    \right\},
\]
where $v(S) = \min_{\operatorname{supp}(\bbeta)\subseteq S,\ \|\bbeta\|_\infty\le M} F(\bX\bbeta)+\lambda_2\|\bbeta\|_2^2$ is the optimal loss on a given support, and $\Phi^\star$ is the optimal value of~\eqref{eq:intro_sparse_glm}.
We use a support-level definition because it gives a finite collection of near-optimal sparse GLMs.

The same BnB tree can collect this set by changing only the pruning threshold.
In ordinary optimization, node $\calN$ is pruned when its safe lower bound satisfies $\mathrm{LB}_{\calN}\ge\mathrm{UB}^{\mathrm{global}}$, where $\mathrm{UB}^{\mathrm{global}}$ is the incumbent loss.
For Rashomon collection, we instead prune node $\calN$ only when
\[
    \mathrm{LB}_{\calN}
    >
    \tau_\epsilon^{\mathrm{RSet}},
    \quad
    \text{where}
    \quad
    \tau_\epsilon^{\mathrm{RSet}}
    =
    (1+\epsilon)\mathrm{UB}^{\mathrm{global}}.
\]
Whenever re-optimization returns a feasible model with objective less than or equal to $\tau_\epsilon^{\mathrm{RSet}}$, we store its support and coefficients.
As the incumbent improves, $\tau_\epsilon^{\mathrm{RSet}}$ decreases and the stored pool is filtered.
At termination, $\mathrm{UB}^{\mathrm{global}}=\Phi^\star$, so the remaining pool is the certified support-level Rashomon set.

In practice, an overly large $\epsilon$ can make $\calR_\epsilon^{\mathrm{supp}}$ too large to enumerate.
We therefore also allow collecting only the best $N$ solutions in the Rashomon set.
Let $\hat v_{(N)}$ be the objective value of the current $N$th-best stored support, with $\hat v_{(N)}=+\infty$ before $N$ supports have been found.
The active pruning threshold becomes
\[
    \tau_{\epsilon,N}^{\mathrm{RSet}}
    =
    \min\left\{
    (1+\epsilon)\mathrm{UB}^{\mathrm{global}},
    \hat v_{(N)}
    \right\}.
\]
The pool keeps only the $N$ best supports found so far.
If the full $\epsilon$-Rashomon set contains at most $N$ supports, this cap has no effect and the method still certifies the complete set.
If the cap is active, termination certifies the best $N$ solutions in the support-level Rashomon set, rather than the entire $\epsilon$-Rashomon set.
Appendix~\ref{sec:compact_rashomon_pool_storage} gives the compact trie-and-offset storage method.

%% file: figures/gpu_bnb_blueprint.tex
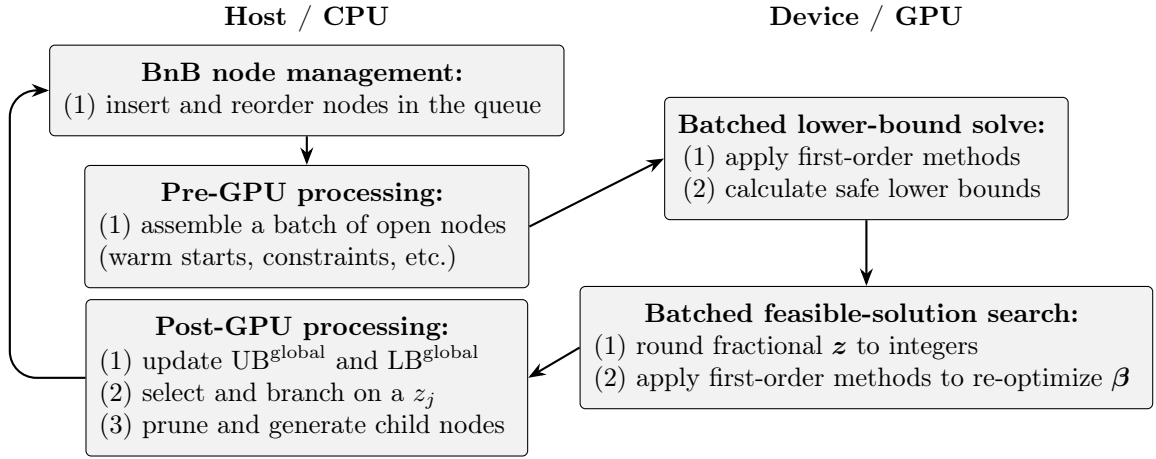
\begin{figure}[!h]
\centering
\begin{tikzpicture}[
  x=1cm,
  y=1cm,
  >=Stealth,
  every node/.style={font=\small},
  block/.style={draw, rounded corners=2pt, align=center, minimum width=4.5cm, minimum height=0.95cm, inner sep=5pt},
  hostblock/.style={block, fill=gray!10},
  gpublock/.style={block, fill=gray!10}
]
  \node[font=\bfseries\small] at (0,3.75) {Host / CPU};
  \node[font=\bfseries\small] at (7.4,3.75) {Device / GPU};

  \node[hostblock] (queue) at (0,2.8) {
    \begin{tabular}{@{}c@{}}
      \textbf{BnB node management:}\\
      \begin{tabular}{@{}l@{}}
        (1) insert and reorder nodes in the queue
      \end{tabular}
    \end{tabular}
    };
  \node[hostblock] (assemble) at (0,1.0) {
    \begin{tabular}{@{}c@{}}
      \textbf{Pre-GPU processing:}\\
      \begin{tabular}{@{}l@{}}
        (1) assemble a batch of open nodes\\
        (warm starts, constraints, etc.)
      \end{tabular}
    \end{tabular}
    };
  \node[hostblock] (post) at (0,-1.0) {
    \begin{tabular}{@{}c@{}}
      \textbf{Post-GPU processing:}\\
      \begin{tabular}{@{}l@{}}
        (1) update $\mathrm{UB}^{\mathrm{global}}$ and $\text{LB}^{\text{global}}$\\
        (2) select and branch on a $z_j$\\
        (3) prune and generate child nodes
      \end{tabular}
    \end{tabular}
    };

  \node[gpublock] (lower) at (7.4,1.9) {
    \begin{tabular}{@{}c@{}}
      \textbf{Batched lower-bound solve:}\\
      \begin{tabular}{@{}l@{}}
        (1) apply first-order methods\\
        (2) calculate safe lower bounds
      \end{tabular}
    \end{tabular}
    };
  \node[gpublock] (reopt) at (7.4,-0.6) {
    \begin{tabular}{@{}c@{}}
      \textbf{Batched feasible-solution search:}\\
      \begin{tabular}{@{}l@{}}
        (1) round fractional $\bz$ to integers\\
        (2) apply first-order methods to re-optimize $\bbeta$
      \end{tabular}
    \end{tabular}
    };

  \draw[->, thick] (queue.south) -- ++(0,-0.32) -- (assemble.north);
  \draw[->, thick] (assemble.east) -- (lower.west);
  \draw[->, thick] (lower.south) -- ++(0,-0.32) -- (reopt.north);
  \draw[->, thick] (reopt.west) -- (post.east);
  \draw[->, thick, rounded corners=10pt] (post.west) -- ++(-1.0,0) |- (queue.west);
\end{tikzpicture}
\caption{
    Hybrid CPU--GPU BnB framework for mixed-integer nonlinear programs, where $\bz$ are integer variables and $\bbeta$ are continuous variables.
    }
\label{fig:gpu_bnb_blueprint}
\end{figure}

%% file: figures/padded_column_sorting_example.tex
\begin{figure}[!ht]
\centering
\begin{tikzpicture}[
    every node/.style={font=\fontsize{9.6pt}{11.5pt}\selectfont},
    panel/.style={
        draw,
        rounded corners=2pt,
        fill=gray!6,
        inner sep=4pt,
        align=center
    },
    title/.style={font=\bfseries\fontsize{9.6pt}{11.5pt}\selectfont}
]
    \node[panel] (sets) at (0.05,2.75) {
        \begin{tabular}{c}
            \textbf{Example batch}: $p=6$, $m=3$, $k=3$\\[2pt]
            \begin{tabular}{c|c|c|c}
                node & $\calJ_0$ & $\calJ_1$ & $\calJ_f$\\
                \hline
                $\calN_1$ & $\{5\}$ & $\{2\}$ & $\{1,3,4,6\}$\\
                $\calN_2$ & $\{2,6\}$ & $\{1,4\}$ & $\{3,5\}$\\
                $\calN_3$ & $\{3\}$ & $\emptyset$ & $\{1,2,4,5,6\}$
            \end{tabular}
        \end{tabular}
    };

    \node[panel] (pad) at (-3.35,-0.25) {
        \begin{tabular}{c}
            \textbf{Padded magnitudes}\\[2pt]
            $\displaystyle
            \bK^{\mathrm{pad}}
            =
            \begin{bmatrix}
                |u_{1,1}| & -\infty & |u_{1,3}| \\
                -\infty   & -\infty & |u_{2,3}| \\
                |u_{3,1}| & |u_{3,2}| & -\infty \\
                |u_{4,1}| & -\infty & |u_{4,3}| \\
                -\infty   & |u_{5,2}| & |u_{5,3}| \\
                |u_{6,1}| & -\infty & |u_{6,3}|
            \end{bmatrix}
            $
        \end{tabular}
    };

    \node[panel] (sort) at (3.45,-0.25) {
        \begin{tabular}{c}
            \textbf{After descending column sort}\\[2pt]
            $\displaystyle
            \bK^{\mathrm{sort}}
            =
            \begin{bmatrix}
                |u_{4,1}| & |u_{5,2}| & |u_{6,3}| \\
                |u_{1,1}| & |u_{3,2}| & |u_{1,3}| \\
                |u_{6,1}| & -\infty & |u_{4,3}| \\
                |u_{3,1}| & -\infty & |u_{2,3}| \\
                -\infty & -\infty & |u_{5,3}| \\
                -\infty & -\infty & -\infty
            \end{bmatrix}
            $
        \end{tabular}
    };
\end{tikzpicture}
\caption{
    Example of padding and column sorting for a batched proximal evaluation.
    Non-free coordinates receive the sentinel value $-\infty$, so a standard descending column sort pushes them below the valid free magnitudes.
}
\label{fig:padded_column_sorting}
\end{figure}

%% file: sections/5_Experiments.tex
\section{Experiments}\label{sec:experiments}

We design the experiments to answer four questions:
(a) how fast is the proposed GPU-parallel BnB solver relative to existing methods for~\eqref{eq:intro_sparse_glm}?
(b) how does the batch size affect total BnB runtime?
(c) how is the runtime distributed across lower-bound computation, feasible-solution search, CPU--GPU data transfer, node generation, and queue management?
(d) how can the collected Rashomon set support variable-importance analysis and model selection?
Results for the latter two questions (c and d) are reported in Appendix~\ref{appendix_sec:additional_experimental_results}.

We compare against commercial and open-source MIP solvers for cardinality-constrained linear and logistic regression.
The baselines are Gurobi~\citep{gurobi}, MOSEK~\citep{mosek}, and OKGLM~\citep{liu2025scalable,liu2026gpufriendlylinearlyconvergentfirstorder}.
OKGLM is the current state-of-the-art open-source implementation: it processes one BnB node at a time, computes lower bounds on the GPU, and selects feasible solutions and branching variables on the CPU.
Detailed experimental settings are given in Appendix~\ref{appendix_sec:experimental_setups}.

\subsection{How Fast Is GPU-Parallel BnB?}

This experiment compares running time, optimality gap, and total number of BnB nodes on both synthetic and real-world instances.
The synthetic benchmark contains highly correlated ($\rho=0.9$) linear- and logistic-regression problems, while the real-world benchmark uses two high-dimensional datasets ($n\ll p$): Santander for linear regression and DOROTHEA for logistic regression.

\input{results/neurips_2026_main_speed_comparison_table}
\input{results/neurips_2026_realworld_speed_table}

Tables~\ref{tab:neurips_main_speed_comparison} and~\ref{tab:neurips_realworld_speed} show that our method is the only method that certifies zero optimality gap on every reported instance.
On synthetic linear-regression problems, our method is consistently the fastest method, reducing the runtime of the serial GPU baseline OKGLM by roughly one order of magnitude and outperforming the commercial MIP solvers by much larger margins on the high-dimensional cases.
The gains are even more pronounced for synthetic logistic regression: Gurobi, MOSEK, and OKGLM~\footnote{OKGLM uses early stopping in some lower-bound solves, which can lead to larger final optimality gaps than Gurobi and MOSEK.} frequently hit the time limit or run out of memory, while our method certifies all instances within the time limit.
For the hardest synthetic logistic case with $p=500$, our method processes $3.76$ million BnB nodes and closes the gap in $4{,}348$ seconds, whereas OKGLM reaches the time limit after processing only $45{,}117$ nodes and still has a large gap.
On the real-world Santander instances, our method is again uniformly faster than OKGLM and closes all gaps.
On DOROTHEA, OKGLM remains competitive for the easiest cases where the tree is very small, but our method becomes faster as $k$ increases and the BnB search becomes large enough for batching to amortize GPU and queue-management overheads.

\subsection{How Does Batch Size Affect BnB Runtime?}

We next study how the GPU batch size affects total certification time.
Using synthetic linear and logistic instances with $n=p=1{,}000$ and $\rho=0.9$, we run only our method and vary the number of BnB nodes processed together in each lower-bound computation/re-optimization batch.
The goal is to measure how larger batches improve GPU throughput.

Figure~\ref{fig:neurips_batch_size_sweep} shows a clear batching effect.
For both linear and logistic regression, increasing the batch size sharply reduces certification time at first, indicating that many BnB nodes can be processed together before GPU throughput becomes saturated.
On the log--log plot, the early part of each curve is close to linear, meaning that doubling the batch size gives an approximately multiplicative runtime reduction.
The benefit eventually saturates: for linear regression the curve flattens around batch size $2^{10}$, while for logistic regression it flattens around batch size $2^{15}$.
This plateau is expected because very large batches cannot always be filled by the current open-node queue, and because kernel throughput, memory traffic, and search adaptivity no longer scale linearly once the GPU workload is already sufficiently large.

\begin{figure}[!t]
\centering
\includegraphics[width=0.9\linewidth]{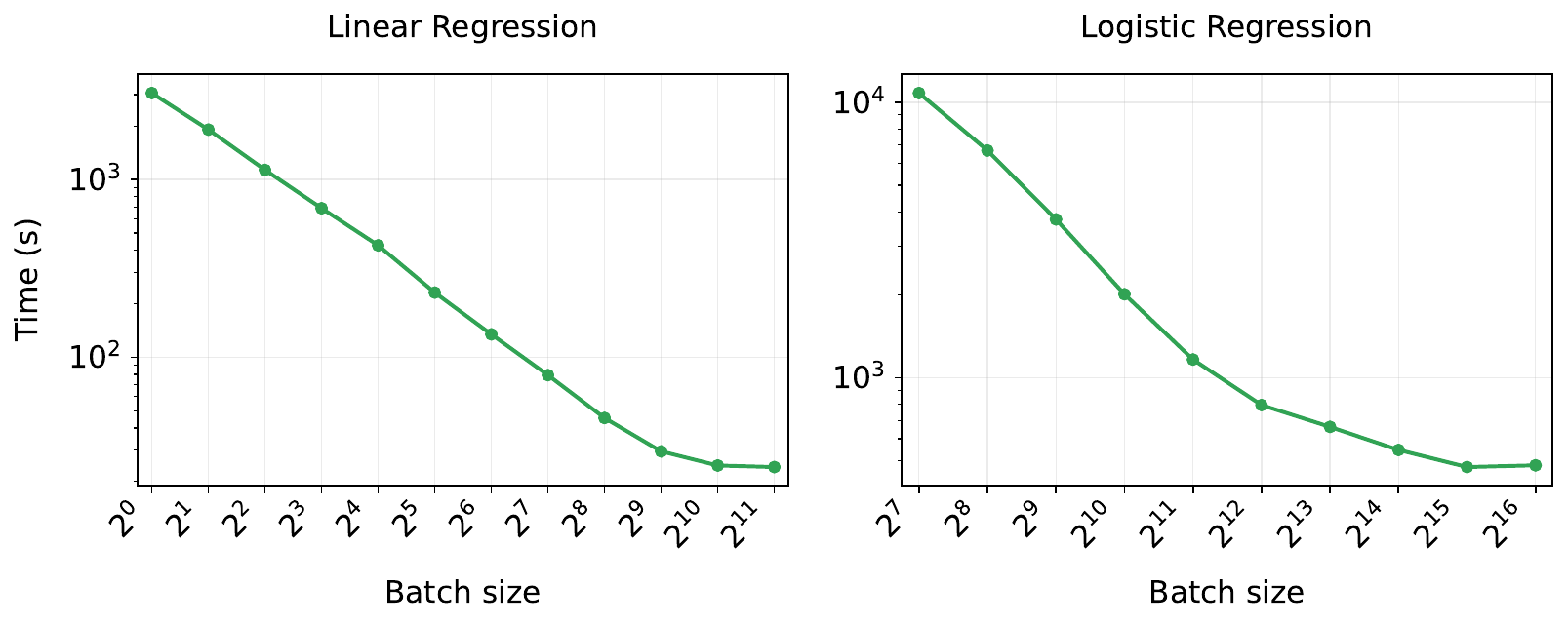}
\caption{Effect of batch-size for GPU-parallel BnB for our method. Results are on synthetic datasets with feature correlation $\rho=0.9$, $n=p=1{,}000$, $k=10$, $\lambda_2=1.0$, and $M=2.0$}
\label{fig:neurips_batch_size_sweep}
\end{figure}

%% file: results/neurips_2026_main_speed_comparison_table.tex
\begin{table}[!t]
\centering
\caption{Results on synthetic datasets with feature correlation $\rho=0.9$, $n=p$, $k=10$, $\lambda_2=1.0$, and $M=2.0$. \texttt{TL} and \texttt{OOM} denote time-limit (10800s) and out-of-memory outcomes.}
\label{tab:neurips_main_speed_comparison}
\resizebox{\linewidth}{!}{%
\begin{tabular}{l*{4}{rrr}}
\toprule
$p$ & \multicolumn{3}{c}{Gurobi} & \multicolumn{3}{c}{MOSEK} & \multicolumn{3}{c}{OKGLM} & \multicolumn{3}{c}{Ours} \\
\cmidrule(lr){2-4} \cmidrule(lr){5-7} \cmidrule(lr){8-10} \cmidrule(lr){11-13}
 & Time (s) & Gap (\%) & Nodes & Time (s) & Gap (\%) & Nodes & Time (s) & Gap (\%) & Nodes & Time (s) & Gap (\%) & Nodes \\
\midrule
\multicolumn{13}{l}{\textit{Synthetic (Linear regression)}} \\
16K & \texttt{TL} & 100 & 1 & \texttt{OOM} & -- & -- & 228.8 & 0.00 & 173 & \textbf{30.6} & 0.00 & 195 \\
8K & \texttt{TL} & 100 & 1 & \texttt{TL} & 13.78 & 6 & 109.5 & 0.00 & 263 & \textbf{15.1} & 0.00 & 257 \\
4K & 9717.0 & 0.00 & 645 & 10498.0 & 0.00 & 419 & 87.3 & 0.00 & 451 & \textbf{16.6} & 0.00 & 565 \\
2K & 3198.0 & 0.00 & 2,732 & 7858.0 & 0.00 & 2,091 & 464.4 & 0.00 & 3,045 & \textbf{20.7} & 0.00 & 2,291 \\
1K & 1845.0 & 0.00 & 7,842 & 2783.0 & 0.00 & 6,073 & 1122.0 & 0.00 & 8,121 & \textbf{24.0} & 0.00 & 7,065 \\
500 & 264.2 & 0.00 & 6,677 & 454.9 & 0.00 & 5,147 & 1065.0 & 0.00 & 8,409 & \textbf{22.2} & 0.00 & 6,089 \\
\midrule
\multicolumn{13}{l}{\textit{Synthetic (Logistic regression)}} \\
16K & \texttt{TL} & 32.52 & 55,063 & \texttt{OOM} & -- & -- & 7790.0 & 0.00 & 3,821 & \textbf{100.8} & 0.00 & 3,865 \\
8K & \texttt{TL} & 38.81 & 160,008 & \texttt{OOM} & -- & -- & \texttt{TL} & 23.63 & 10,885 & \textbf{93.5} & \textbf{0.00} & 12,939 \\
4K & \texttt{TL} & 50.05 & 552,994 & \texttt{TL} & 10.63 & 1,057 & 10361.0 & 0.00 & 25,477 & \textbf{80.4} & 0.00 & 26,861 \\
2K & \texttt{TL} & 48.42 & 518,798 & \texttt{TL} & 10.31 & 4,824 & \texttt{TL} & 55.49 & 29,691 & \textbf{160.7} & \textbf{0.00} & 122,299 \\
1K & \texttt{TL} & 41.93 & 521,731 & \texttt{TL} & 9.42 & 38,517 & \texttt{TL} & 70.07 & 38,610 & \textbf{473.5} & \textbf{0.00} & 742,719 \\
500 & \texttt{TL} & 27.54 & 821,898 & \texttt{TL} & 6.74 & 228,302 & \texttt{TL} & 69.60 & 45,117 & \textbf{4348.0} & \textbf{0.00} & 3,763,479 \\
\bottomrule
\end{tabular}%
}
\vspace{1mm}
\end{table}

%% file: results/neurips_2026_realworld_speed_table.tex
\begin{table}[!t]
\centering
\caption{Results on real-world datasets: Santander (linear regression) and DOROTHEA (logistic regression). Santander has $n=4459$, $p=4735$, $\lambda_2=1.0$, $M=10$, and $k\in\{6,7,8,9,10\}$; DOROTHEA has $n=2300$, $p=89989$, $\lambda_2=1.0$, $M=10$, and $k\in\{5,15,\ldots,45\}$. \texttt{TL} and \texttt{OOM} denote time-limit (10800s) and out-of-memory outcomes.}
\label{tab:neurips_realworld_speed}
\resizebox{\linewidth}{!}{%
\begin{tabular}{l*{4}{rrr}}
\toprule
$k$ & \multicolumn{3}{c}{Gurobi} & \multicolumn{3}{c}{MOSEK} & \multicolumn{3}{c}{OKGLM} & \multicolumn{3}{c}{Ours} \\
\cmidrule(lr){2-4} \cmidrule(lr){5-7} \cmidrule(lr){8-10} \cmidrule(lr){11-13}
 & Time (s) & Gap (\%) & Nodes & Time (s) & Gap (\%) & Nodes & Time (s) & Gap (\%) & Nodes & Time (s) & Gap (\%) & Nodes \\
\midrule
\multicolumn{13}{l}{\textit{Santander (Linear regression)}} \\
6 & \texttt{TL} & 100 & 1 & \texttt{TL} & 0.27 & 725 & 111.4 & 0.00 & 1,395 & \textbf{17.7} & 0.00 & 951 \\
7 & \texttt{TL} & 100 & 1 & \texttt{TL} & 0.43 & 705 & 195.3 & 0.00 & 2,209 & \textbf{21.7} & 0.00 & 1,693 \\
8 & \texttt{TL} & 100 & 1 & \texttt{TL} & 0.26 & 617 & 419.1 & 0.00 & 4,541 & \textbf{26.7} & 0.00 & 3,443 \\
9 & \texttt{TL} & 100 & 1 & \texttt{TL} & 0.30 & 575 & 1034.0 & 0.00 & 10,445 & \textbf{35.5} & 0.00 & 7,381 \\
10 & \texttt{TL} & 100 & 1 & \texttt{TL} & 0.52 & 473 & 3901.0 & 0.00 & 35,505 & \textbf{52.3} & 0.00 & 22,121 \\
\midrule
\multicolumn{13}{l}{\textit{DOROTHEA (Logistic regression)}} \\
5 & 868.1 & 0.00 & 938 & 1074.0 & 0.00 & 0 & \textbf{19.1} & 0.00 & 11 & 34.4 & 0.00 & 11 \\
15 & 2901.0 & 0.00 & 3,209 & \texttt{OOM} & -- & -- & 61.4 & 0.00 & 33 & \textbf{58.8} & 0.00 & 37 \\
25 & \texttt{TL} & 0.12 & 3,901 & \texttt{OOM} & -- & -- & 338.5 & 0.00 & 177 & \textbf{224.1} & 0.00 & 259 \\
35 & \texttt{TL} & 0.16 & 3,996 & \texttt{OOM} & -- & -- & 4380.0 & 0.00 & 2,271 & \textbf{904.7} & 0.00 & 2,983 \\
45 & \texttt{TL} & 0.17 & 3,847 & \texttt{OOM} & -- & -- & \texttt{TL} & 0.06 & 2,206 & \textbf{2198.0} & \textbf{0.00} & 15,873 \\
\bottomrule
\end{tabular}%
}
\end{table}

%% file: sections/6_Conclusion.tex
\section{Conclusion}\label{sec:conclusion}

We introduced a simple, generic, and modular CPU--GPU BnB framework for certifying cardinality-constrained GLMs.
The framework turns repeated node-level computations into batched GPU work by combining padding with GPU-efficient routines.
This design keeps each BnB component independent and also extends naturally to exact support-level Rashomon-set collection.
Empirically, our method closes all reported optimality gaps and achieves one-to-two orders of magnitude speedups on challenging synthetic and real-world instances.

%% file: sections/7_Appendix.tex
\section*{Appendix}\label{sec:appendix}

\section{Primer on Branch and Bound}
\label{appendix_sec:bnb_primer}

\input{sections/Appendix/bnb_primer.tex}

\section{A Slightly More Efficient PAVA Algorithm}
\label{appendix_sec:more_efficient_PAVA}

\input{sections/Appendix/more_efficient_PAVA.tex}

\clearpage
\section{Proofs}
\label{appendix_sec:proofs}

\input{sections/Appendix/proofs.tex}

\clearpage
\section{Multi-GPU Distributed Computing}
\label{appendix_sec:multi_gpu_distributed_computing}

\input{sections/Appendix/multi_gpu_distributed_computing.tex}

\clearpage
\section{Experimental Setups}
\label{appendix_sec:experimental_setups}

\input{sections/Appendix/experimental_setups.tex}

\clearpage
\section{Additional Experimental Results}
\label{appendix_sec:additional_experimental_results}

\input{sections/Appendix/additional_experiments.tex}

\clearpage
\section{Compact Rashomon-Set storage}
\label{sec:compact_rashomon_pool_storage}

\input{sections/Appendix/compact_Rashomon_pool_storage.tex}

%% file: sections/Appendix/bnb_primer.tex
A branch-and-bound tree stores partial fixing decisions on the binary support indicators $\bz$.
At each node, the solver computes a valid lower bound, compares it with the incumbent objective value, and either prunes the node or branches on another free indicator.
Figure~\ref{fig:bnb_primer} gives a visualization of this process.

\begin{figure}[!h]
    \centering
    \begin{minipage}{0.55\linewidth}
        \centering
        \begin{tikzpicture}[
            scale=0.8,
            every node/.style={font=\small},
            bnode/.style={circle, draw=blue!55!black, very thick, fill=gray!25, minimum size=7.5mm, inner sep=0pt},
            leafbad/.style={circle, draw=blue!55!black, very thick, fill=gray!10, minimum size=7.5mm, inner sep=0pt},
            leafgood/.style={circle, draw=blue!55!black, very thick, fill=green!65!black, minimum size=7.5mm, inner sep=0pt},
            edge/.style={-{Latex[length=2.0mm]}, thick}
        ]
            \node[bnode] (root) at (0,0) {};
            \node[bnode] (l1) at (-1.85,-1.05) {};
            \node[bnode] (r1) at (1.85,-1.05) {};
            \node[bnode] (l2) at (-3.05,-2.25) {};
            \node[leafbad] (l2r) at (-0.85,-2.25) {};
            \node[leafbad] (r2) at (0.85,-2.25) {};
            \node[leafgood] (r3) at (3.05,-2.25) {};
            \node[leafbad] (l3) at (-4.05,-3.75) {};
            \node[leafbad] (l4) at (-2.05,-3.75) {};

            \draw[edge] (root) -- node[above left=-1pt] {$z_2=0$} (l1);
            \draw[edge] (root) -- node[above right=-1pt] {$z_2=1$} (r1);
            \draw[edge] (l1) -- node[above left=-1pt] {$z_1=0$} (l2);
            \draw[edge] (l1) -- node[above right=-1pt] {$z_1=1$} (l2r);
            \draw[edge] (r1) -- node[above left=-1pt] {$z_7=0$} (r2);
            \draw[edge] (r1) -- node[above right=-1pt] {$z_7=1$} (r3);
            \draw[edge] (l2) -- node[above left=-1pt] {$z_5=0$} (l3);
            \draw[edge] (l2) -- node[above right=-1pt] {$z_5=1$} (l4);

            \draw[red!85, very thick] (l2r.south west) -- (l2r.north east);
            \draw[red!85, very thick] (l2r.north west) -- (l2r.south east);
            \draw[red!85, very thick] (r2.south west) -- (r2.north east);
            \draw[red!85, very thick] (r2.north west) -- (r2.south east);
            \draw[red!85, very thick] (l3.south west) -- (l3.north east);
            \draw[red!85, very thick] (l3.north west) -- (l3.south east);
            \draw[red!85, very thick] (l4.south west) -- (l4.north east);
            \draw[red!85, very thick] (l4.north west) -- (l4.south east);
        \end{tikzpicture}
    \end{minipage}
    \hfill
    \begin{minipage}{0.42\linewidth}
        \caption{
            Primer on branch and bound.
            Each open node stores partial fixing decisions on $\bz$.
            A lower bound prunes a node (with red crosses) when it cannot beat the incumbent; otherwise the node is branched into two children by fixing one free variable, either with $z_j = 0$ or $z_j = 1$
            The green node finds an incumbent, which becomes optimal when there are no open nodes.
            In this work, we process multiple nodes in batches on the GPU.
        }
        \label{fig:bnb_primer}
    \end{minipage}
\end{figure}
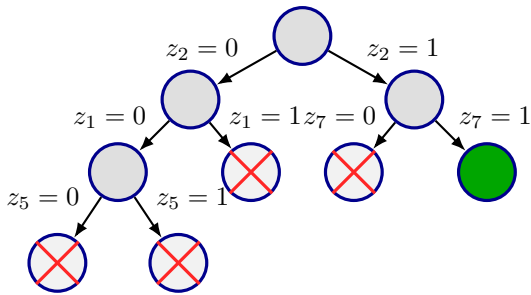

%% file: sections/Appendix/more_efficient_PAVA.tex
By definition, the proximal operator of $g_{\calN}^*$ is
\begin{align*}
    \hat{\balpha} = \argmin_{\balpha \in \bbR^p} \frac{1}{2}\|\balpha - \bbeta\|_2^2 + \rho \operatorname{TopSum}_k \left( H_M \left( \balpha \right) \right)
\end{align*}

~\citet{liu2026gpufriendlylinearlyconvergentfirstorder} shows that we can recast $\operatorname{prox}_{\rho g^*_{\calN}}$ as a generalized isotonic regression problem.
Without loss of generality, let us assume that $\calN$ is the root node.
We have
\begin{align*}
    \hat{\balpha} &= \operatorname{sgn} \left( \bbeta \right) \odot \pi^{-1} \left( \bv \right), \\
    \text{where} \quad \hat{\bv} &= \argmin_{\bv \in \bbR^p} \sum_{j=1}^p \frac{1}{2} \left( v_j - |\beta_{\pi\left( j \right) }| \right)^2 + \rho_j H_M \left( v_j \right)
    \quad
    \st
    \quad
    v_1 \geq v_2 \geq \ldots \geq v_p \geq 0,
\end{align*}
where $\rho_j = 1$ if $j \leq k$ and $\rho_j = 0$ if $j > k$.
Algorithm~\ref{alg:prox_of_g_conjugate_root_node} is the PAVA procedure to compute $\hat{\balpha}$.
The algorithm merges adjacent blocks according to the block up-and-down procedure in order to get rid of the violation of monotonicity constraint.
\input{figures/pava_algorithm.tex}

However, there is additional structure we can exploit to have a slightly more efficient implementation than the vanilla PAVA procedure.

Note that $\rho_j = 1$ for all $j \leq k$ and $\rho_j = 0$ for all $j > k$.
Moreover, $|\beta_{\pi \left( j \right) }|$ has already been sorted.
Therefore, when we initialize $\hat{\alpha}_j$ at Line 3 of Algorithm~\ref{alg:prox_of_g_conjugate_root_node}, we have
\begin{align*}
    \hat{\alpha}_1 \geq \hat{\alpha}_2 \geq \ldots \geq \hat{\alpha}_k \quad \text{and} \quad \hat{\alpha}_{k+1} \geq \hat{\alpha}_{k+2} \geq \ldots \geq \hat{\alpha}_p.
\end{align*}
Therefore a violation of the required nonincreasing order can only begin at the boundary between positions $k$ and $k+1$.
Thus, we can start checking the violation of the monotonicity constraint at this boundary, expand the active pooled block only when it violates the next right or left singleton value, and stop once both neighboring inequalities are satisfied.
This keeps the worst-case linear cost of PAVA, but avoids storing a full stack of pooled blocks for every new instance.
In our batched setting on the GPU, each column only needs the active interval endpoints and a few scalar block summaries.

%% file: figures/pava_algorithm.tex
\begin{algorithm}[!h]
    \DontPrintSemicolon
    \caption{Compute $\prox_{\rho g^*_{\mathcal{N}}}(\bbeta)$ at the root node $\mathcal{N}$}
    \label{alg:prox_of_g_conjugate_root_node}
    \KwData{vector $\bbeta \in \mathbb R^p$, scalar $\rho > 0$, cardinality parameter $k \in [p]$, box parameter $M > 0$}
    \KwResult{$\prox_{\rho g^*_{\mathcal{N}}}(\bbeta)$}

    Set $\brho \in \mathbb R_+^{p}$ with $\rho_j \gets \rho$ if $j \in \{1, 2, ..., k \}$ and $\rho_j \gets 0$ otherwise. \;
    Sort $\bbeta$ with permutation $\bm \pi$ of $[p]$ such that $\vert{\beta_{\pi(1)}} \geq \ldots \geq \vert{\beta_{\pi(p)}} \geq 0$\;

    \lFor{$j \gets 1$ \KwTo $p$}{
        $\hat{v}_j \gets \prox_{\rho_j H_M} \left( \beta_{\pi(j)} \right) $
    }
    $\mathcal{W} \leftarrow \{[1, 1], [2, 2]\ldots, [p, p]\}$. \;
    \tcp{Constraint violations can be checked more efficiently by expanding the active endpoints from $k$ and $k+1$ to the left and right, respectively.}
    \While{$\exists [j_1, j_2], [j_2+1, j_3] \in \mathcal{W} ~ \st ~ \hat{v}_{j_1} < \hat{v}_{j_3}$}{
        $\mathcal{W} \gets \mathcal{W} \setminus \{[j_1, j_2]\} \setminus \{[j_2+1, j_3]\} $ \;
        $\overline{\rho} \gets \left(\sum_{j=j_1}^{j_3} \rho_j \right) / \left( j_3 - j_1 + 1 \right)$ \;
        $\overline{ \xi } \gets \left(\sum_{j=j_1}^{j_3} | \beta_{\pi(j)} |\right) / \left( j_3 - j_1 + 1 \right)$\;
        $\hat{v}_{[j_1:j_3]} \gets \prox_{\overline{\rho} H_M} \left( \overline{ \xi } \right)$\;
        $ \mathcal{W} \leftarrow \mathcal{W} \cup \{[j_1, j_3]\} $\;
    }
    \KwRet{$\operatorname{sgn}(\bbeta) \odot \bpi^{-1}(\hat{\bv})$}
\end{algorithm}

%% file: sections/Appendix/proofs.tex
\subsection{Proof of Theorem~\ref{thm:recover_relaxed_indicators}}
\label{app:proof_recover_relaxed_indicators}

\begin{namedtheorem}[Theorem~\ref{thm:recover_relaxed_indicators}]
Fix a BnB node $\calN$ and let $\bbeta^\star$ solve problem~\eqref{eq:node_composite}.
Let $\bar k=k-|\calJ_1(\calN)|$ and $p_f=|\calJ_f(\calN)|$.
Set $z_j^\star=0$ for $j\in\calJ_0(\calN)$ and $z_j^\star=1$ for $j\in\calJ_1(\calN)$.
On the free set $\calJ_f(\calN)$, if $\bar k=0$, set $z_j^\star=0$.
If at most $\bar k$ free coefficients are nonzero, set $z_j^\star=1$ for nonzero $\beta_j^\star$ and $z_j^\star=0$ for zero $\beta_j^\star$.
Otherwise, sort the free magnitudes as
$|\beta_{\pi(1)}^\star|\ge\cdots\ge|\beta_{\pi(p_f)}^\star|$.
Find an index $s\in\{0,\ldots,\bar k-1\}$ such that
\[
    \tau
    :=
    \frac{\sum_{r=s+1}^{p_f}|\beta_{\pi(r)}^\star|}{\bar k-s},
    \qquad
    |\beta_{\pi(s)}^\star|\ge \tau \ge |\beta_{\pi(s+1)}^\star|,
\]
with the convention $|\beta_{\pi(0)}^\star|=+\infty$.
Then set
\[
    z_{\pi(r)}^\star
    =
    \begin{cases}
        1, & r\le s,\\
        |\beta_{\pi(r)}^\star|/\tau, & r>s,
    \end{cases}
    \qquad r=1,\ldots,p_f.
\]
This gives an optimal relaxed indicator vector paired with $\bbeta^\star$.
\end{namedtheorem}

\begin{proof}
For fixed $\bbeta^\star$, the smooth term $F(\bX\bbeta^\star)$ is constant with respect to $\bz$.
Thus recovering $\bz^\star$ reduces to solving the optimization problem defining $g_{\calN}(\bbeta^\star)$.
The coordinates in $\calJ_0(\calN)$ and $\calJ_1(\calN)$ are fixed by the node.
On the free coordinates, the remaining problem is
\[
    \min_{\bz}
    \left\{
        \frac12\sum_{j\in\calJ_f(\calN)}
        \frac{(\beta_j^\star)^2}{z_j}
        :
        0\le z_j\le 1,\quad
        |\beta_j^\star|\le Mz_j,\quad
        \sum_{j\in\calJ_f(\calN)}z_j\le \bar k
    \right\}.
\]
When $\beta_j^\star=0$, the corresponding objective term is defined to be $0$ even if $z_j=0$.
Thus zero-coefficient coordinates do not affect the minimization over $\bz$.
For the nonzero free coordinates, we can combine the two lower bounds for $z_j$ and get a unified lower bound $z_j\ge |\beta_j^\star|/M$.
The resulting $\bz$-subproblem is convex with linear constraints, so any feasible point satisfying the KKT conditions is globally optimal.
Let $\pi$ sort the free magnitudes in descending order.\\

\textbf{Nonbinding remaining budget.}

If $\bar k\ge p_f$, or if $\bar k<p_f$ and $|\beta_{\pi(\bar k+1)}^\star|=0$, then the number of nonzero free coefficients is at most $\bar k$.
Setting $z_j=1$ for all $j\in\calJ_f(\calN)$ with $\beta_j^\star\ne0$ and $z_j=0$ for all $j\in\calJ_f(\calN)$ with $\beta_j^\star=0$ satisfies the cardinality constraint.
This choice is optimal because $(\beta_j^\star)^2/(2z_j)$ is decreasing in $z_j$ whenever $\beta_j^\star\ne0$, so every nonzero free coordinate should use the largest feasible value $z_j=1$ when the budget allows it.\\

\textbf{Binding remaining budget.}

Now suppose $\bar k<p_f$ and $|\beta_{\pi(\bar k+1)}^\star|>0$.
There are then more than $\bar k$ nonzero free coefficients, so setting $z_j=1$ for all of them is infeasible.
Moreover, the cardinality constraint must be active at the optimum.
If $\sum_{j\in\calJ_f(\calN)}z_j<\bar k$, then at least one nonzero free coordinate must have $z_j<1$; increasing that coordinate slightly would remain feasible and would strictly decrease the objective, a contradiction.
Thus
\[
    \sum_{j\in\calJ_f(\calN)}z_j=\bar k.
\]

We translate the binding-budget problem into its Lagrangian form.
Let $\nu\ge0$ be the multiplier for $\sum_{j\in\calJ_f(\calN)}z_j\le \bar k$, let $\eta_j\ge0$ be the multiplier for $|\beta_j^\star|/M-z_j\le0$, and let $\omega_j\ge0$ be the multiplier for $z_j-1\le0$.
Then the Lagrangian for the above optimization problem is
\begin{align*}
    \calL(\bz,\nu,\boldsymbol\eta,\boldsymbol\omega)
    &=
    \frac12\sum_{j\in\calJ_f(\calN)}
    \frac{(\beta_j^\star)^2}{z_j}
    +
    \nu\left(\sum_{j\in\calJ_f(\calN)}z_j-\bar k\right)\\
    &\quad+
    \sum_{j\in\calJ_f(\calN)}
    \eta_j(\frac{|\beta_j^\star|}{M}-z_j)
    +
    \sum_{j\in\calJ_f(\calN)}
    \omega_j(z_j-1).
\end{align*}
For the nonzero free coordinates, the relevant KKT conditions are:
\begin{align*}
    &\text{primal feasibility:}
    && |\beta_j^\star|/M\le z_j\le 1
    \quad \forall j\in\calJ_f(\calN)\ \text{with}\ \beta_j^\star\ne0,\\
    &&& \sum_{j\in\calJ_f(\calN)}z_j\le \bar k,\\
    &\text{dual feasibility:}
    && \nu\ge0,\quad \eta_j\ge0,\quad \omega_j\ge0
    \quad \forall j\in\calJ_f(\calN)\ \text{with}\ \beta_j^\star\ne0,\\
    &\text{stationarity:}
    && -\frac{(\beta_j^\star)^2}{2z_j^2}+\nu-\eta_j+\omega_j=0
    \quad \forall j\in\calJ_f(\calN)\ \text{with}\ \beta_j^\star\ne0,\\
    &\text{complementary slackness:}
    && \nu\!\left(\sum_{j\in\calJ_f(\calN)}z_j-\bar k\right)=0,\\
    &&& \eta_j(|\beta_j^\star|/M-z_j)=0
    \quad \forall j\in\calJ_f(\calN)\ \text{with}\ \beta_j^\star\ne0, \\
    &&& \omega_j(z_j-1)=0
    \quad \forall j\in\calJ_f(\calN)\ \text{with}\ \beta_j^\star\ne0.
\end{align*}
The multiplier $\nu$ is strictly positive in this case.
Let $q$ be the number of nonzero free coefficients.
The binding-budget case has $q>\bar k$.
If every nonzero free coordinate had $z_j=1$, then $\sum_{j\in\calJ_f(\calN)}z_j\ge q>\bar k$, which violates feasibility.
Therefore, at least one nonzero free coordinate $j_0$ must satisfy $z_{j_0}<1$.
For this coordinate, complementary slackness for the upper bound gives $\omega_{j_0}=0$.
Its stationarity equation becomes
\[
    \nu
    =
    \frac{(\beta_{j_0}^\star)^2}{2z_{j_0}^2}
    +
    \eta_{j_0}.
\]
Since $\beta_{j_0}^\star\ne0$, $z_{j_0}>0$, and $\eta_{j_0}\ge0$, the right-hand side is strictly positive.
Thus $\nu>0$.

For any nonzero coordinate whose lower and upper bounds are inactive (\textit{i.e.}, $|\beta_j^\star|/M<z_j<1$, so $\eta_j=\omega_j=0$), stationarity gives
\[
    z_j=\frac{|\beta_j^\star|}{\tau},
\]
where $\tau:=\sqrt{2\nu}$.
Moreover, since we know $|\beta_j^\star|/M<z_j$, we get $\tau < M$.

For any nonzero coordinate whose lower bound is active (\textit{i.e.}, $z_j=|\beta_j^\star|/M$), we get $\tau \geq M$.
Here we only need to discuss coordinates whose lower bound is active while the upper bound is inactive; if $z_j=|\beta_j^\star|/M=1$, then the coordinate is already covered by the capped case $z_j=1$.
To see this, note that the stationarity condition with $\omega_j=0$ (because upper bound is inactive) gives
\[
    \eta_j
    =
    \nu-\frac{M^2}{2}
    =
    \frac{\tau^2-M^2}{2}.
\]
Since $\eta_j\ge0$, lower-bound activity requires $\tau\ge M$.

Thus, from the previous discussions on Case 1 (there exist some coordinate such that both the lower and upper bounds are inactive) and Case 2 (there exists some coordinate such that the lower bound is active), we can conclude that Case 1 and Case 2 cannot coexist.

Case 1 leads to the scenario that every nonzero free coordinate is either upper-bound active, so $z_j=1$, or follows the inactive-bound stationarity rule $z_j=|\beta_j^\star|/\tau$.
Case 2 corresponds to the degenerate threshold value $\tau=M$: a lower-bound-active coordinate has $z_j=|\beta_j^\star|/M$, which is the same as $z_j=|\beta_j^\star|/\tau$ when $\tau=M$.
Therefore, both cases can be represented by the single capped form
\[
    z_j=\min\left\{1,\frac{|\beta_j^\star|}{\tau}\right\},
\]
with $\tau\in(0,M]$.
For this capped form, we can choose nonnegative KKT multipliers $(\nu,\boldsymbol\eta,\boldsymbol\omega)$ with $\nu=\tau^2/2$ so that stationarity and complementary slackness hold coordinatewise.
More explicitly, for every nonzero free coordinate, one valid choice is
\[
    \nu=\frac{\tau^2}{2},
    \qquad
    \eta_j=0,
    \qquad
    \omega_j
    =
    \begin{cases}
        \big((\beta_j^\star)^2-\tau^2\big)/2, & |\beta_j^\star|\ge \tau,\\
        0, & |\beta_j^\star|<\tau.
    \end{cases}
\]
If $|\beta_j^\star|\ge\tau$, then $z_j=1$ and this choice of $\omega_j$ enforces stationarity.
If $|\beta_j^\star|<\tau$, then $z_j=|\beta_j^\star|/\tau$ and stationarity holds with $\eta_j=\omega_j=0$.
The lower-bound-active case occurs only at the boundary $\tau=M$, where $z_j=|\beta_j^\star|/M=|\beta_j^\star|/\tau$, so $\eta_j=0$ is still valid.

Because the cardinality constraint is binding, the remaining task is to choose $\tau$ so that the capped formula uses exactly the remaining budget.
Equivalently, $\tau$ is chosen as the solution of the scalar equation
\begin{equation}
    \sum_{j\in\calJ_f(\calN)}
    \min\left\{1,\frac{|\beta_j^\star|}{\tau}\right\}
    =
    \bar k,
    \qquad
    0<\tau\le M.
    \label{eq:tau_budget_equation}
\end{equation}
Such a $\tau$ exists.
Indeed, feasibility of $\bbeta^\star$ for the relaxation implies
\[
    \sum_{j\in\calJ_f(\calN)}
    \frac{|\beta_j^\star|}{M}
    \le
    \bar k.
\]
Therefore, the left-hand side above is at most $\bar k$ when $\tau=M$.
On the other hand, as $\tau\downarrow0$, the same left-hand side approaches the number of nonzero free coefficients, which is larger than $\bar k$ in the binding-budget case.
By continuity, a solution $\tau\in(0,M]$ exists for~\eqref{eq:tau_budget_equation}.

Equivalently, after sorting the free coordinates so that
\[
    |\beta_{\pi(1)}^\star|
    \ge
    |\beta_{\pi(2)}^\star|
    \ge
    \cdots
    \ge
    |\beta_{\pi(p_f)}^\star|,
\]
we find the number $s$ of coordinates that are capped at $z_j=1$.
For a candidate $s\in\{0,\ldots,\bar k-1\}$, the budget equation becomes
\[
    s+\frac{1}{\tau}
    \sum_{r=s+1}^{p_f}|\beta_{\pi(r)}^\star|
    =
    \bar k,
\]
which provides us with a formula to compute $\tau$ as
\[
    \tau
    =
    \frac{\sum_{r=s+1}^{p_f}|\beta_{\pi(r)}^\star|}{\bar k-s}.
\]
A coordinate is capped exactly when $|\beta_j^\star|\ge \tau$.
Thus, if exactly the first $s$ sorted coordinates are capped, then the $s$th sorted magnitude must be at least $\tau$, while the next sorted magnitude must be at most $\tau$.
Therefore, the correct value of $s$ is any value satisfying the consistency condition
\[
    |\beta_{\pi(s)}^\star|
    \ge
    \tau
    \ge
    |\beta_{\pi(s+1)}^\star|,
\]
with the boundary convention $|\beta_{\pi(0)}^\star|=+\infty$.
Once such an $s$ is found, the formula
\[
    z_j=\min\left\{1,\frac{|\beta_j^\star|}{\tau}\right\}
\]
recovers an optimal relaxed indicator vector.
\end{proof}

%% file: sections/Appendix/multi_gpu_distributed_computing.tex
There are two useful multi-GPU regimes.
The node-parallel regime is straightforward: the CPU keeps the global BnB frontier, assigns different node batches to different GPUs, and collects the returned bounds, feasible solutions, branching candidates, and node statuses.
Each GPU then runs the same single-GPU pipeline on its assigned nodes.

The more interesting regime is row-distributed data-parallel computation, used when $\bX$ is too big to fit on one GPU.
Split the data into $D$ row groups,
\begin{align}
    \bX=
    \begin{bmatrix}\bX^{(1)}\\ \vdots\\ \bX^{(D)}\end{bmatrix},
    \qquad
    \by=
    \begin{bmatrix}\by^{(1)}\\ \vdots\\ \by^{(D)}\end{bmatrix},
    \qquad
    \bX^{(d)}\in\bbR^{n_d\times p},
    \qquad
    \sum_{d=1}^D n_d=n.
    \label{eq:distributed_row_partition}
\end{align}
For a batch coefficient matrix $\bB$, GPU $d$ evaluates its local predictors, derivatives, dual variables, and feature-space products as
\begin{align}
    \bS^{(d)}&=\bX^{(d)}\bB,
    &
    R_{i,b}^{(d)}&=\partial_s\ell(s,y_i^{(d)})\big|_{s=S_{i,b}^{(d)}},
    &
    \bZeta^{(d)}&=-\bR^{(d)},\nonumber\\
    \bG^{(d)}&=(\bX^{(d)})^\top\bR^{(d)},
    &
    \bQ^{(d)}&=\frac{1}{2\lambda_2}(\bX^{(d)})^\top\bZeta^{(d)}.
    \label{eq:distributed_local_quantities}
\end{align}
The global gradient and scaled feature-space dual matrix are obtained by summing over row groups as
\begin{align}
    \nabla_{\bB}\calF_{\mathrm{batch}}(\bB)
    =
    \sum_{d=1}^D \bG^{(d)},
    \qquad
    \bQ
    =
    \sum_{d=1}^D \bQ^{(d)}.
    \label{eq:distributed_gradient_and_dual_reductions}
\end{align}
The same local-evaluation and cross-GPU-summation pattern gives the primal and dual objective vectors.
For each node $b$, GPU $d$ computes
\begin{align}
    \varphi_b^{(d)}
    =
    \sum_{i=1}^{n_d}\ell(S_{i,b}^{(d)},y_i^{(d)}),
    \qquad
    \psi_b^{(d)}
    =
    \sum_{i=1}^{n_d}\ell^*(-\zeta_{i,b}^{(d)},y_i^{(d)}),
    \label{eq:distributed_local_objective_terms}
\end{align}
and the coordinator forms
\begin{align}
    [\bPhi(\bB)]_b
    &=
    \sum_{d=1}^D \varphi_b^{(d)}
    +
    2\lambda_2 g_{\calN_b}(\bbeta^{(b)}),
    \label{eq:distributed_primal_objective}\\
    [\bm{\Psi}(\bZeta^{(1)},\ldots,\bZeta^{(D)})]_b
    &=
    -\sum_{d=1}^D \psi_b^{(d)}
    -
    2\lambda_2 g_{\calN_b}^*(\bq^{(b)}).
    \label{eq:distributed_dual_objective}
\end{align}
In summary, row distribution changes only the row-dependent computations: each GPU evaluates its own smooth-loss, conjugate-loss, gradient, and feature-space dual contributions, and a coordinator GPU sums these quantities over $d=1,\ldots,D$.
The coordinator then applies the same feature-side proximal operator and node-dependent kernels as in the single-GPU batched algorithm, and broadcasts the updated $\bB$ for the next iteration.

%% file: sections/Appendix/experimental_setups.tex
\subsection{Datasets}

\paragraph{Synthetic Data Generation Process}

For each synthetic instance, we set $n=p$ and generate the rows of $\bX$ independently from a centered Gaussian distribution with Toeplitz covariance,
\[
    \bx_i\sim \calN(\mathbf 0,\bSigma),
    \qquad
    \Sigma_{j\ell}=\rho^{|j-\ell|}.
\]
The parameter $\rho$ controls feature correlation, with larger values producing more strongly correlated columns.
We construct the true sparse coefficient vector $\bbeta^\star$ by setting every $(p/k)$th coordinate to $1$, and setting all other coordinates to zero.
In other words, nonzero entries are placed at evenly spaced coordinates with the first one starting at the $(p/k)$-th coordinate.

For linear regression, responses are generated from
\[
    y_i=\bx_i^\top\bbeta^\star+\epsilon_i,
\]
where $\epsilon_i$ is Gaussian noise with $\epsilon_i\sim\calN(0,\|\bX\bbeta^\star\|/\mathrm{SNR})$ and $\mathrm{SNR}=5$.
For logistic regression, labels are sampled from $\{-1,1\}$ according to
\[
    \mathbb P(y_i=1\mid \bx_i)
    =
    \frac{1}{1+\exp(-\bx_i^\top\bbeta^\star)}.
\]
Throughout the experiments, we use cardinality constraint $k=10$, $\ell_2$ regularization $\lambda_2=1.0$, box constraint $M=2.0$, and feature correlations $\rho=0.9$ for both linear and logistic regression.
We choose the feature dimension $p$ from the set $\{16000, 8000, 4000, 2000, 1000, 500\}$.
The smaller the feature dimension is, the harder the problem becomes to certify optimality (requiring processing many more nodes in BnB) because the number of observations is decreasing.
For reproducibility, we use the same random seed to generate the synthetic dataset, so all methods will run on the same data instance.

\paragraph{Real-world Datasets and Preprocessing}

We also evaluate on two real-world datasets following the OKGLM experiments~\citep{liu2026gpufriendlylinearlyconvergentfirstorder}.
For linear regression, we use the Santander Customer Transaction Prediction dataset~\citep{santander}.
After removing redundant features and normalizing the remaining columns, the processed instance has $n=4459$ observations and $p=4735$ features.
For logistic regression, we use the DOROTHEA drug-discovery dataset~\citep{asuncion2007uci}.
This dataset is a high-dimensional binary classification benchmark built from molecular descriptors.
We use a balanced version of the data; after removing redundant features, the processed instance has $n=2300$ observations and $p=89989$ features.
For both real-world datasets, each feature column is centered to have mean $0$ and rescaled to have Euclidean norm $1$.

For the real-world experiments, we set $\lambda_2=1.0$ and $M=10$.
These choices follow the earlier cross-validation study in the OKGLM experiments, where $\lambda_2=1.0$ performed best on both datasets and $M=10$ was large enough to keep the box constraint from affecting the selected sparse models.
For Santander, we report results for $k\in\{6,7,8,9,10\}$.
For DOROTHEA, we report results for $k\in\{5,15,25,35,45\}$.

\subsection{Baselines}

Gurobi and MOSEK (both use academic licenses) are applied to perspective formulations of Problem~\eqref{eq:intro_sparse_glm}.
For linear regression, Gurobi uses the native perspective MIP formulation.
For logistic regression, Gurobi uses the perspective formulation together with outer-approximation (cutting planes) for the logistic loss function.
MOSEK uses the perspective formulation for both losses.
Both commercial solvers receive the same beam-search warm start used by the original OKGLM implementation.

OKGLM (BSD-3 license) processes one BnB node at a time with beam-search size $5$ and lower-bound method uses duality-gap-restarted $+$ accelerated proximal gradient method.
Our method uses the same lower-bound computation method, but processes multiple nodes per GPU batch and performs rounding and re-optimization in batches on the GPU.
Unless otherwise stated, our method chooses the largest safe batch size for the active GPU and problem size.
See Appendix~\ref{sub:batch_size_policy} for more details.

\subsection{Additional Setup Details for Experiments}

We record total running time, final optimality gap, number of processed BnB nodes, effective batch size, runtime component breakdown, and GPU-utilization summaries.

For both the baselines and our method, we set a time limit to 3 hours.

The batch-size experiment uses synthetic linear and logistic instances with $n=p=1000$, feature correlation $\rho=0.9$, cardinality constraint $k=10$, $\ell_2$ regularization $\lambda_2=1.0$, and box constraint $M=2.0$.
For logistic regression, we vary the batch size over $\{2^7, \dots, 2^{16}\}$.
For linear regression, which is easier to solve, we also include smaller batch sizes and vary the batch size over $\{2^0,\ldots,2^{11} \}$.

\subsection{Automatic Batch Size Selection}
\label{sub:batch_size_policy}

Our method automatically selects the largest safe batch size for the GPU machine and the dataset size, using a memory-safe heuristic.

The automatic batch-size rule is intentionally conservative.
Let $U$ denote the usable GPU memory after reserving a safety margin, and let $m_{\mathrm{node}}$ denote the estimated per-node workspace required by one batched BnB node.
The realized batch size is chosen as the largest power of two not exceeding $U/m_{\mathrm{node}}$.
The per-node estimate can be decomposed as
\[
    m_{\mathrm{node}}
    =
    m_{\mathrm{lb}}(p)
    +
    m_{\mathrm{re-opt}}(n,k,\ell),
\]
where $m_{\mathrm{lb}}(p)$ is the lower-bound workspace and $m_{\mathrm{re-opt}}(n,k,\ell)$ is the re-optimization workspace for loss $\ell$.
For a fixed $(n,p,k)$, the lower-bound term is the same for linear and logistic regression, but the re-optimization term of logistic regression is larger than that of linear regression.
Re-optimization for linear regression mostly repeats simple least-squares calculations on the selected features. Re-optimization for logistic regression is heavier because, for each candidate support, it repeatedly computes prediction scores, converts them into probabilities, and evaluates the logistic objective. These extra arrays require more memory per node, so logistic batches are often smaller at the same $(n,p,k)$.
Thus two instances with the same $n$, $p$, and $k$ can receive different automatic batch sizes.
For example, in the synthetic $p=500$ main comparison, the raw safe capacity is slightly above $2^{16}$ for linear regression but slightly below $2^{16}$ for logistic regression; after power-of-two rounding, the realized batch sizes become $65536$ and $32768$, respectively.

\subsection{Computing Platforms}

We ran the GPU experiments on a computing cluster using NVIDIA A100 GPU nodes.
Unless otherwise stated, each GPU run uses one A100 GPU, and our method selects the largest safe batch size for the active GPU and instance size.
The commercial MIP baselines are run on CPU nodes with AMD Milan processors at 2.45 GHz; each baseline run uses 8 CPU cores and 100GB memory.

%% file: sections/Appendix/additional_experiments.tex
\subsection{How Much Time Does Each Component in Our BnB Take?}
\label{appendix_sub:BnB_component_time}

For our method, we report component-level wall-clock time for the batched lower-bound solve, feasible-solution re-optimization, CPU--GPU data transfer, branching, and node generation.
Tables~\ref{tab:ours_synthetic_profile} and~\ref{tab:ours_realworld_profile} report these statistics for the synthetic and real-world experiments in Section~\ref{sec:experiments}.
Each timing cell shows seconds on the first line and the percentage of total wall-clock time on the second line.
The lower-bound batch count is the number of batched GPU lower-bound passes; the re-optimization batch count is the number of batched re-optimization passes.
These counts are not exactly equal to the processed-node count divided by the effective batch size because the first batch, the last batches, and intermediate batches generated after pruning need not be full.

\input{results/neurips_2026_ours_profile_synthetic_compact_table}

\input{results/neurips_2026_ours_profile_realworld_compact_table}

\clearpage
\subsection{Variable Importance Analysis based on the Rashomon Set}
\label{appendix_sub:rashomon_variable_importance}

Let the saved Rashomon set be
\[
    \{(S_m,\widehat\bbeta^{(m)},\Phi_m):m=1,\ldots,N\}.
\]
The most straightforward variable importance analysis is support frequency:
\[
    \widehat\pi_j
    =
    \frac1N\sum_{m=1}^N\ind\{j\in S_m\}.
\]
Features with $\widehat\pi_j\approx1$ are selected by nearly all near-optimal sparse GLMs.
Features with intermediate frequency may be substitutable with correlated alternatives.

Coefficient summaries can be computed by defining $\widehat\beta_j^{(m)}=0$ when $j\notin S_m$ and reporting means, absolute means, sign frequencies, and coefficient ranges over the pool.\\

We can also perform variable importance analysis based on the model reliance score.
For logistic regression, a simple fixed-model reliance score removes the fitted contribution of feature $j$ while holding all other coefficients fixed.
Let
\[
    \eta_i^{(m)}=\bx_i^\top\widehat\bbeta^{(m)},
    \qquad
    \eta_{i,-j}^{(m)}=\eta_i^{(m)}-x_{ij}\widehat\beta_j^{(m)}.
\]
With labels $y_i\in\{-1,+1\}$, define
\[
    R_j^{(m)}
    =
    \frac1n\sum_{i=1}^n
    \log(1+\exp(-y_i\eta_{i,-j}^{(m)}))
    -
    \frac1n\sum_{i=1}^n
    \log(1+\exp(-y_i\eta_i^{(m)})).
\]
The interval
\[
    \left[\min_m R_j^{(m)},\max_m R_j^{(m)}\right]
\]
summarizes how much feature $j$ can matter across the entire sparse GLM Rashomon pool.
This is the GLM analogue of studying variable importance across many good models instead of one selected model \citep{fisher2019all,dong2020exploring,donnelly2023rashomon}.

Figures~\ref{fig:rashomon_support_frequency_logistic_rho09} and~\ref{fig:rashomon_model_reliance_logistic_rho09} summarize the saved Rashomon pool for the synthetic logistic instance with $n=p=1000$, $\rho=0.9$, $k=10$, $\lambda_2=1.0$, $M=2.0$, $\epsilon=0.1$, and $N_\texttt{Rashomon}=1000$.
The support-frequency plot counts how often each displayed feature appears in the saved sparse supports and orders the displayed features by increasing feature index.
The model-reliance plot reports, for each displayed feature, the increase in mean training logistic loss after dropping that feature's fitted contribution from each saved model, again ordered by increasing feature index.

Both plots show that the high-frequency and high-reliance features align closely with the true signal features at indices $0,100,200,\ldots,1000$.
The main discrepancy is that feature $599$ appears in place of feature $600$, but the two receive similar importance scores, suggesting that they act as nearly interchangeable correlated predictors.
Thus, the Rashomon pool reveals the broader set of statistically plausible features, rather than only the support selected by minimizing the objective function.

\begin{figure}[!h]
\centering
\includegraphics[width=0.95\linewidth]{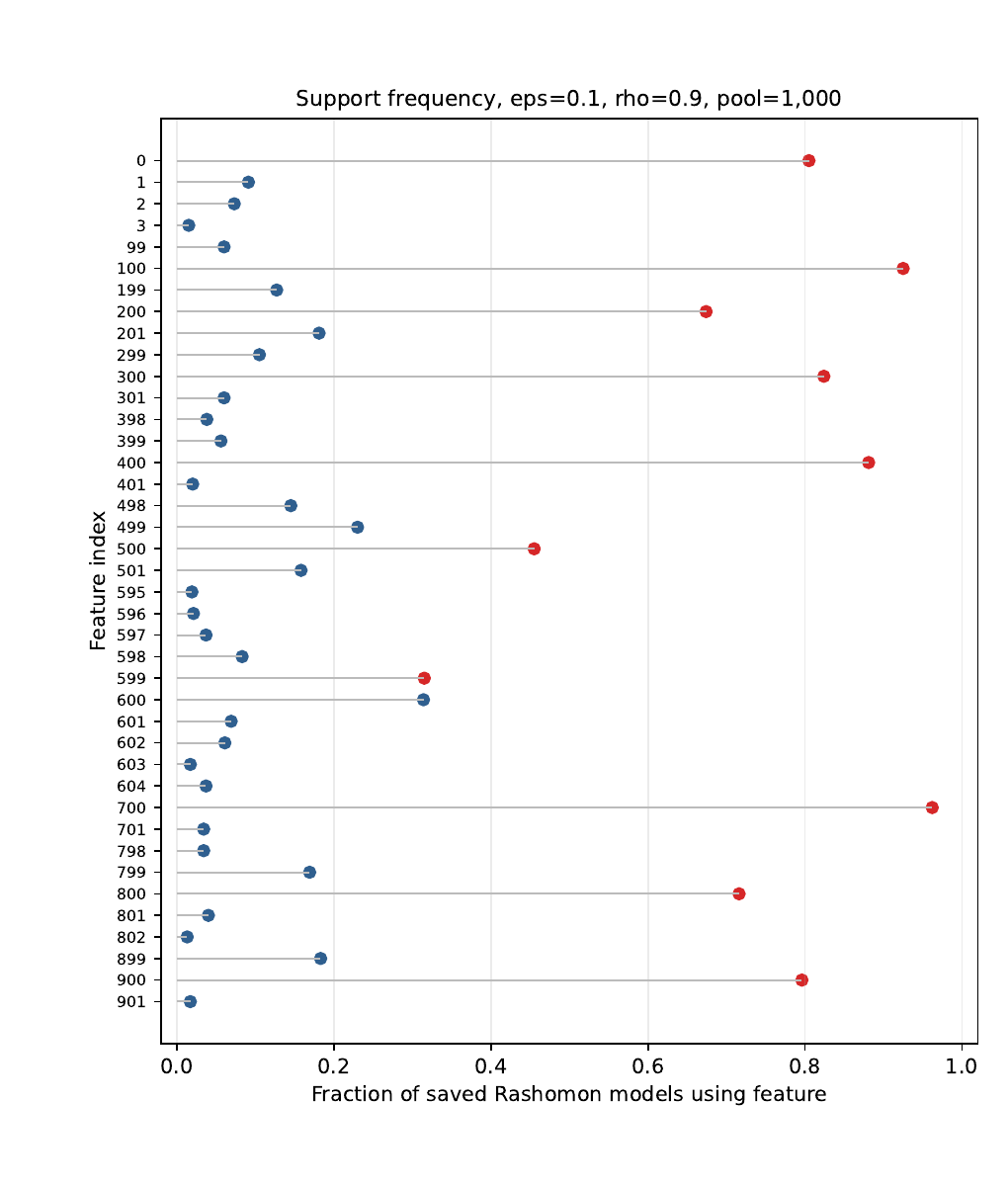}
\caption{Support-frequency summary for the saved top-$1000$ Rashomon pool on the synthetic logistic instance with $n=p=1000$, $\rho=0.9$, $k=10$, $\lambda_2=1.0$, and $M=2.0$. The displayed features are features that appear most frequently in the Rashomon set, listed vertically in increasing feature-index order. Red markers denote features used by the best saved sparse logistic model.}
\label{fig:rashomon_support_frequency_logistic_rho09}
\end{figure}

\begin{figure}[!h]
\centering
\includegraphics[width=0.82\linewidth]{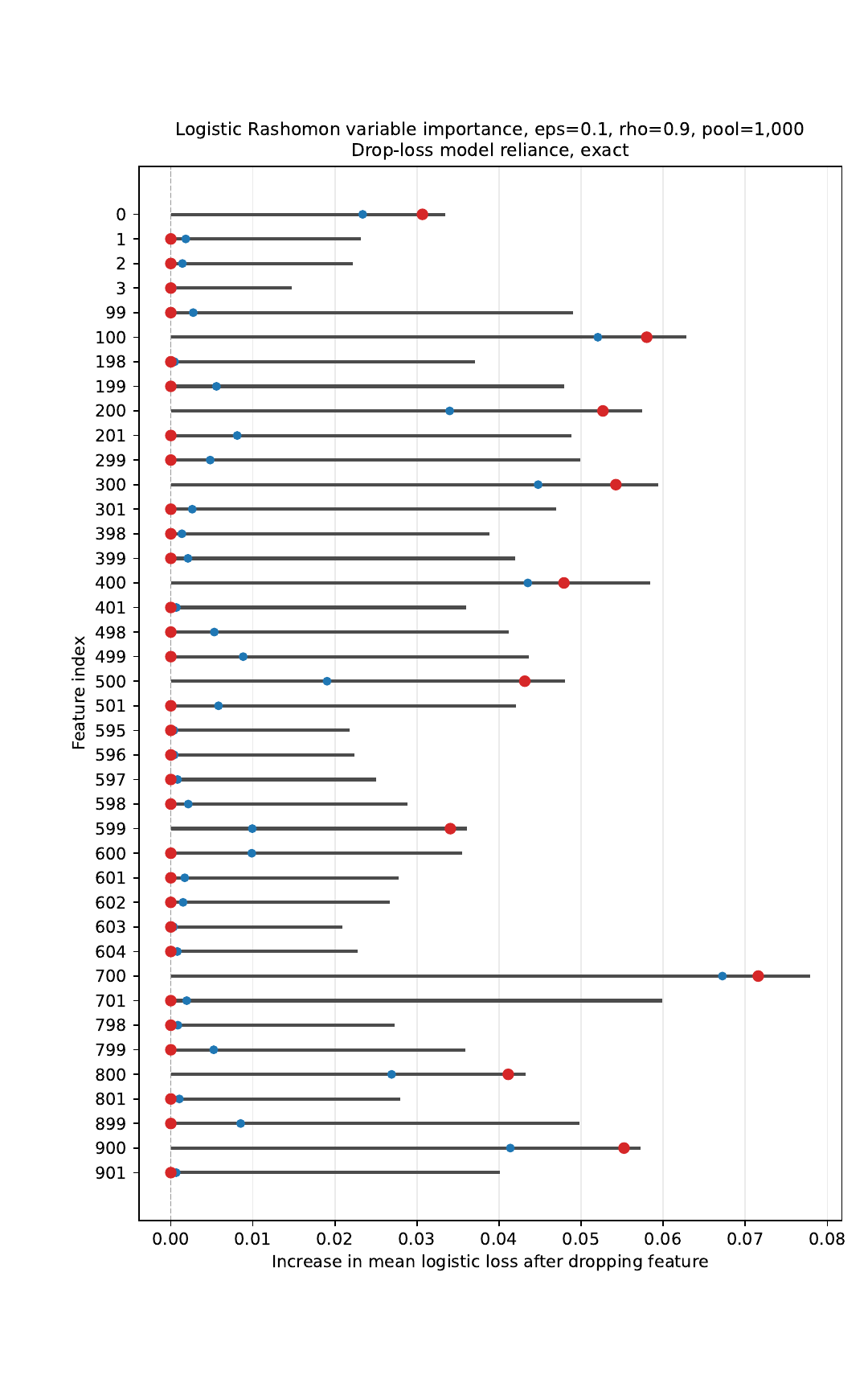}
\caption{Model-reliance summary for the saved top-$1000$ Rashomon pool on the synthetic logistic instance with $n=p=1000$, $\rho=0.9$, $k=10$, $\lambda_2=1.0$, and $M=2.0$. For each displayed feature, the score is the increase in mean logistic loss after removing the feature contribution while keeping the remaining fitted coefficients fixed. Gray segments show the minimum-to-maximum range over the saved pool, red markers show the best saved sparse logistic model, and blue markers show the pool mean. Features are listed vertically in increasing feature-index order.}
\label{fig:rashomon_model_reliance_logistic_rho09}
\end{figure}

\clearpage
\subsection{Secondary-Metric Consideration on the Dorothea Rashomon Pool}
\label{appendix_sub:dorothea_secondary_metric_queries}

Once we collect the Rashomon set, we can select models based on different metrics by scanning the saved models.
For classification, each model can be evaluated under secondary criteria such as AUC, accuracy, or calibration, allowing practitioners to choose a near-optimal sparse GLM that performs best on the metric they care about without rerunning BnB.
This mirrors the use of tree Rashomon sets for answering many model selection questions~\citep{xin2022exploring}.

Figures~\ref{fig:dorothea_primary_vs_auc}, \ref{fig:dorothea_primary_vs_accuracy}, and~\ref{fig:dorothea_accuracy_vs_auc} evaluate secondary metrics over the saved Dorothea logistic Rashomon set with $k=5$, $\lambda_2=1.0$, $M=10.0$, $\epsilon=0.1$, and \texttt{rashomon\_n}=1000.
The primary objective shown in the plot is the sparse logistic objective including the ridge penalty term.

\begin{figure}[!h]
\centering
\includegraphics[width=0.95\linewidth]{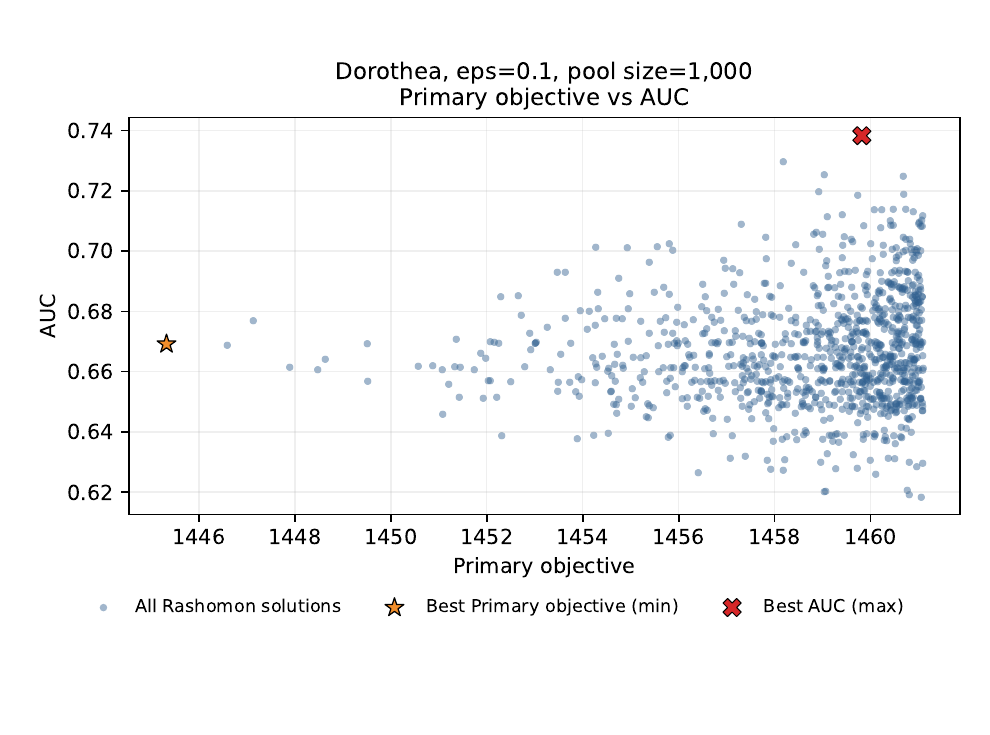}
\caption{Primary sparse-logistic objective versus AUC over the saved top-$1000$ Dorothea Rashomon set with $k=5$. The orange star is the model minimizing the primary objective; the red marker is the model maximizing AUC.}
\label{fig:dorothea_primary_vs_auc}
\end{figure}

\begin{figure}[!h]
\centering
\includegraphics[width=0.95\linewidth]{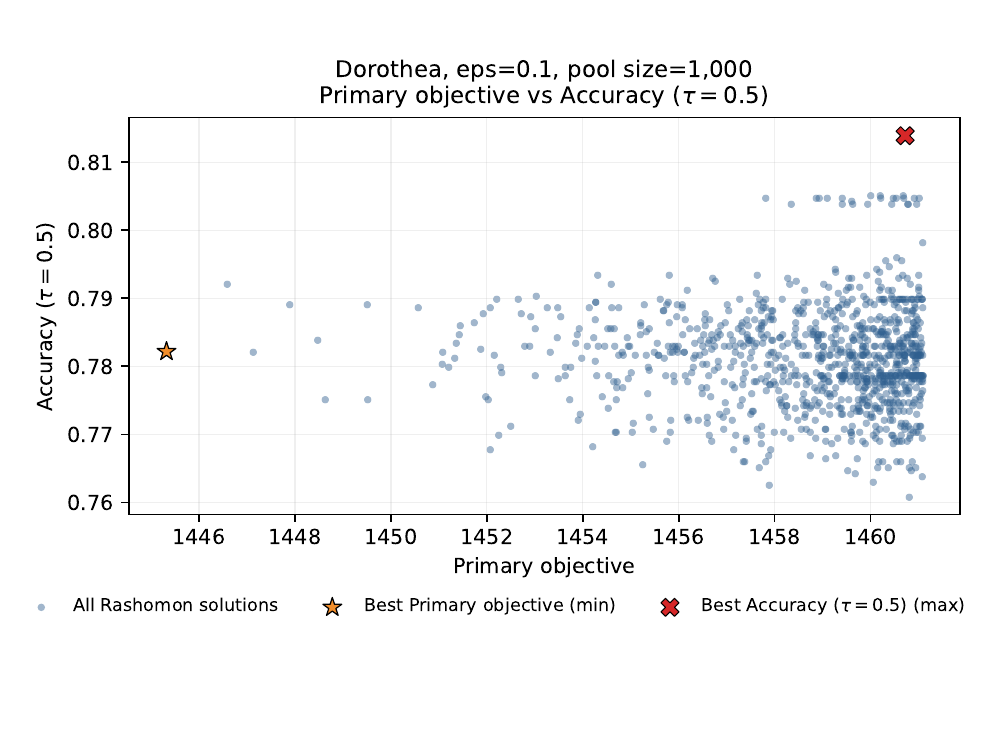}
\caption{Primary sparse-logistic objective versus accuracy over the saved top-$1000$ Dorothea Rashomon set with $k=5$. Accuracy uses the threshold $\tau=0.5$ on predicted probabilities.}
\label{fig:dorothea_primary_vs_accuracy}
\end{figure}

\begin{figure}[!h]
\centering
\includegraphics[width=0.95\linewidth]{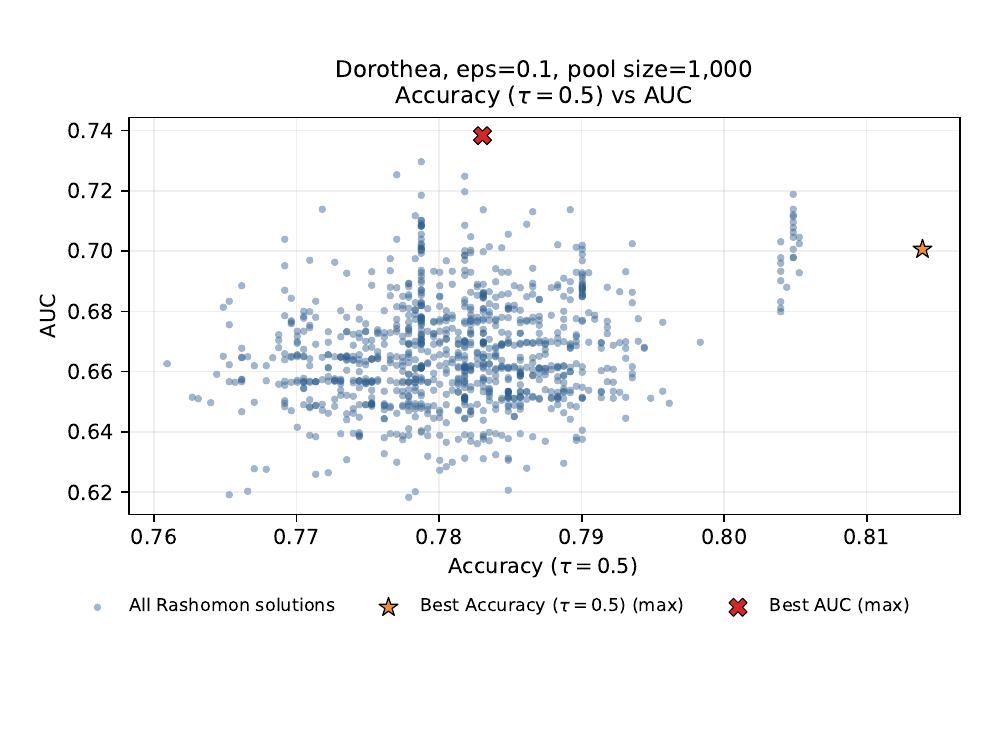}
\caption{Accuracy versus AUC over the saved top-$1000$ Dorothea Rashomon set with $k=5$. This plot shows whether the models preferred by a threshold-dependent metric are also preferred by a ranking metric.}
\label{fig:dorothea_accuracy_vs_auc}
\end{figure}

%% file: results/neurips_2026_ours_profile_synthetic_compact_table.tex
\begin{table}[!h]
\centering
\caption{Component-level runtime statistics for our method on the synthetic datasets. We report wall-clock seconds on the first line and the percentage of total wall time on the second line.}
\label{tab:ours_synthetic_profile}
\resizebox{\linewidth}{!}{%
\begin{tabular}{lcccccccc}
\toprule
$p$ & Total & Lower bound & Re-opt. & Transfer & Branch/gen. & LB batches & Re-opt. batches & Batch size \\
 & (s, \%) & (s, \%) & (s, \%) & (s, \%) & (s, \%) & & & \\
\midrule
\multicolumn{9}{l}{\textit{Synthetic (Linear regression)}} \\
16K & \shortstack{30.6\\100.0\%} & \shortstack{24.7\\80.6\%} & \shortstack{5.8\\19.1\%} & \shortstack{0.0\\0.1\%} & \shortstack{0.0\\0.0\%} & 12 & 12 & 1,024 \\
8K & \shortstack{15.1\\100.0\%} & \shortstack{12.1\\80.2\%} & \shortstack{2.9\\19.3\%} & \shortstack{0.0\\0.1\%} & \shortstack{0.0\\0.1\%} & 11 & 11 & 4,096 \\
4K & \shortstack{16.6\\100.0\%} & \shortstack{10.7\\64.1\%} & \shortstack{5.9\\35.3\%} & \shortstack{0.0\\0.1\%} & \shortstack{0.0\\0.1\%} & 12 & 12 & 8,192 \\
2K & \shortstack{20.7\\100.0\%} & \shortstack{13.1\\63.4\%} & \shortstack{7.3\\35.5\%} & \shortstack{0.0\\0.2\%} & \shortstack{0.0\\0.2\%} & 15 & 14 & 16,384 \\
1K & \shortstack{24.0\\100.0\%} & \shortstack{14.4\\60.0\%} & \shortstack{9.1\\37.9\%} & \shortstack{0.0\\0.2\%} & \shortstack{0.1\\0.4\%} & 18 & 17 & 32,768 \\
500 & \shortstack{22.2\\100.0\%} & \shortstack{13.0\\58.3\%} & \shortstack{9.0\\40.3\%} & \shortstack{0.0\\0.1\%} & \shortstack{0.1\\0.3\%} & 18 & 17 & 65,536 \\
\midrule
\multicolumn{9}{l}{\textit{Synthetic (Logistic regression)}} \\
16K & \shortstack{100.8\\100.0\%} & \shortstack{81.7\\81.1\%} & \shortstack{17.0\\16.9\%} & \shortstack{0.3\\0.3\%} & \shortstack{0.3\\0.3\%} & 15 & 14 & 1,024 \\
8K & \shortstack{93.5\\100.0\%} & \shortstack{68.5\\73.3\%} & \shortstack{21.2\\22.7\%} & \shortstack{0.5\\0.6\%} & \shortstack{0.7\\0.8\%} & 19 & 18 & 4,096 \\
4K & \shortstack{80.4\\100.0\%} & \shortstack{52.7\\65.5\%} & \shortstack{23.3\\29.0\%} & \shortstack{0.6\\0.7\%} & \shortstack{0.9\\1.1\%} & 21 & 20 & 8,192 \\
2K & \shortstack{160.7\\100.0\%} & \shortstack{87.2\\54.3\%} & \shortstack{48.0\\29.9\%} & \shortstack{0.8\\0.5\%} & \shortstack{9.7\\6.0\%} & 28 & 27 & 16,384 \\
1K & \shortstack{473.5\\100.0\%} & \shortstack{190.3\\40.2\%} & \shortstack{153.3\\32.4\%} & \shortstack{1.8\\0.4\%} & \shortstack{64.6\\13.6\%} & 54 & 53 & 32,768 \\
500 & \shortstack{4348.3\\100.0\%} & \shortstack{467.0\\10.7\%} & \shortstack{375.5\\8.6\%} & \shortstack{4.1\\0.1\%} & \shortstack{3189.8\\73.4\%} & 130 & 128 & 32,768 \\
\bottomrule
\end{tabular}%
}
\end{table}

%% file: results/neurips_2026_ours_profile_realworld_compact_table.tex
\begin{table}[!t]
\centering
\caption{Component-level runtime statistics for our method on the real-world Santander and DOROTHEA datasets. We report wall-clock seconds on the first line and the percentage of total wall time on the second line.}
\label{tab:ours_realworld_profile}
\resizebox{\linewidth}{!}{%
\begin{tabular}{lcccccccc}
\toprule
$k$ & Total & Lower bound & Re-opt. & Transfer & Branch/gen. & LB batches & Re-opt. batches & Batch size \\
 & (s, \%) & (s, \%) & (s, \%) & (s, \%) & (s, \%) & & & \\
\midrule
\multicolumn{9}{l}{\textit{Santander (Linear regression)}} \\
6 & \shortstack{17.7\\100.0\%} & \shortstack{15.1\\85.3\%} & \shortstack{2.4\\13.8\%} & \shortstack{0.0\\0.2\%} & \shortstack{0.0\\0.1\%} & 41 & 40 & 8,192 \\
7 & \shortstack{21.7\\100.0\%} & \shortstack{18.5\\85.4\%} & \shortstack{2.9\\13.4\%} & \shortstack{0.0\\0.2\%} & \shortstack{0.0\\0.2\%} & 42 & 41 & 8,192 \\
8 & \shortstack{26.7\\100.0\%} & \shortstack{22.9\\85.6\%} & \shortstack{3.4\\12.6\%} & \shortstack{0.1\\0.3\%} & \shortstack{0.1\\0.3\%} & 45 & 44 & 8,192 \\
9 & \shortstack{35.5\\100.0\%} & \shortstack{30.1\\84.7\%} & \shortstack{4.4\\12.3\%} & \shortstack{0.1\\0.4\%} & \shortstack{0.2\\0.5\%} & 52 & 51 & 8,192 \\
10 & \shortstack{52.3\\100.0\%} & \shortstack{42.3\\80.9\%} & \shortstack{6.7\\12.7\%} & \shortstack{0.3\\0.6\%} & \shortstack{0.7\\1.4\%} & 63 & 62 & 8,192 \\
\midrule
\multicolumn{9}{l}{\textit{DOROTHEA (Logistic regression)}} \\
5 & \shortstack{34.4\\100.0\%} & \shortstack{34.2\\99.6\%} & \shortstack{0.1\\0.4\%} & \shortstack{0.0\\0.0\%} & \shortstack{0.0\\0.0\%} & 6 & 6 & 256 \\
15 & \shortstack{58.8\\100.0\%} & \shortstack{58.0\\98.6\%} & \shortstack{0.7\\1.3\%} & \shortstack{0.0\\0.0\%} & \shortstack{0.0\\0.0\%} & 17 & 17 & 256 \\
25 & \shortstack{224.1\\100.0\%} & \shortstack{221.6\\98.9\%} & \shortstack{1.9\\0.9\%} & \shortstack{0.1\\0.1\%} & \shortstack{0.1\\0.0\%} & 33 & 32 & 256 \\
35 & \shortstack{904.7\\100.0\%} & \shortstack{892.8\\98.7\%} & \shortstack{5.1\\0.6\%} & \shortstack{0.8\\0.1\%} & \shortstack{1.1\\0.1\%} & 52 & 51 & 256 \\
45 & \shortstack{2197.6\\100.0\%} & \shortstack{2129.4\\96.9\%} & \shortstack{17.8\\0.8\%} & \shortstack{3.2\\0.1\%} & \shortstack{10.1\\0.5\%} & 86 & 85 & 256 \\
\bottomrule
\end{tabular}%
}
\end{table}

%% file: sections/Appendix/compact_Rashomon_pool_storage.tex
A support-level Rashomon set may contain many sparse supports.
Storing one dense length-$p$ mask per support is wasteful when $k\ll p$, and storing one independent length-$k$ index array per support misses shared structure across related supports.
We instead store supports in a trie data structure.
Please see Figure~\ref{fig:rashomon_support_trie} for a visualization.
This is similar to the compressed model-set representation in TreeFARMS~\citep{xin2022exploring}, but the object stored here is simpler: TreeFARMS stores sparse decision-tree structures, whereas our pool stores GLM support sets and optional re-optimized coefficients.
Note that the trie for Rashomon-set storage is not the BnB search tree.
The BnB tree has binary branch edges such as $z_j=1$ and $z_j=0$.
Our trie data structure is a separate storage object: after a Rashomon support has been found, it stores only the included feature indices and omits all excluded branch decisions.

Formally, let the trie have node set $\calV$ and root $r_0$.
Each non-root node $v\in\calV\setminus\{r_0\}$ stores a parent $p(v)\in\calV$ and an edge label $a(v)\in[p]$.
For any leaf $\ell$, let
\[
    P(\ell)=(a_1,\ldots,a_q)
\]
be the sequence of edge labels along the path from $r_0$ to $\ell$.
The recovered support is the set $S(\ell)=\{a_1,\ldots,a_q\}$.
The insertion sequence may follow the order of included branch decisions that produced the support; it is a storage order, while $S(\ell)$ is the unordered support set.
Thus a Rashomon record can be written as
\[
    \left(\ell_m,\ \Phi_m\right),
\]
where $\ell_m$ is the trie leaf for the $m$th saved support and $\Phi_m=v(S(\ell_m))$ is its support-restricted objective value.
The support itself is determined only by the trie leaf $\ell_m$.

For example, suppose the pool contains supports
\[
    \{2,7,18\},\qquad
    \{2,11,18\},\qquad
    \{3,9,18\},\qquad
    \{4,12\}.
\]
If their insertion sequences are $(18,2,7)$, $(18,2,11)$, $(18,9,3)$, and $(12,4)$, then the trie shares the prefix $(18,2)$ between the first two records and the prefix $(18)$ between the first three records.
\input{figures/rashomon_support_trie.tex}

If coefficients are also stored, they should be attached to the leaf rather than trie edges.
The same feature can have different fitted coefficients in different supports, so edge-level coefficient sharing is not valid.
Let the saved supports be $S_1,\ldots,S_N$, and let
\[
    \widehat{\bbeta}_{S_m}^{(m)}
    \in
    \bbR^{|S_m|}
\]
denote the re-optimized coefficients on support $S_m$, ordered in the same order as the recovered trie labels.
We store all active coefficients in one vector $\bc$ together with an offset vector $\vo=(o_0,o_1,\ldots,o_N)$, where $o_0=0$ and
\[
    o_m=\sum_{r=1}^m |S_r|.
\]
This layout stores only the active coefficients, requiring $\sum_{m=1}^N |S_m|$ numbers instead of $Np$ dense entries.
It also avoids keeping a separate coefficient vector for every solution in the Rashomon set: once the support is recovered from the trie, the two neighboring offsets in $\vo$ identify exactly where its coefficients are stored in $\bc$.
Then the coefficient vector for record $m$ is the slice
\[
    \widehat{\bbeta}_{S_m}^{(m)}
    =
    \bc_{\,o_{m-1}+1:o_m}.
\]
For the four supports above, if the fitted coefficients are
\[
    (0.4,-1.2,0.3),\quad
    (0.5,-1.0,0.1),\quad
    (0.2,0.8,-0.4),\quad
    (-0.6,1.1),
\]
then
\[
    \vo=(0,3,6,9,11),
    \qquad
    \bc=(0.4,-1.2,0.3,\ 0.5,-1.0,0.1,\ 0.2,0.8,-0.4,\ -0.6,1.1).
\]
The second record, for example, uses entries $4$ through $6$ of $\bc$, matching the path sequence $(18,2,11)$ and the support $\{2,11,18\}$.

%% file: figures/rashomon_support_trie.tex
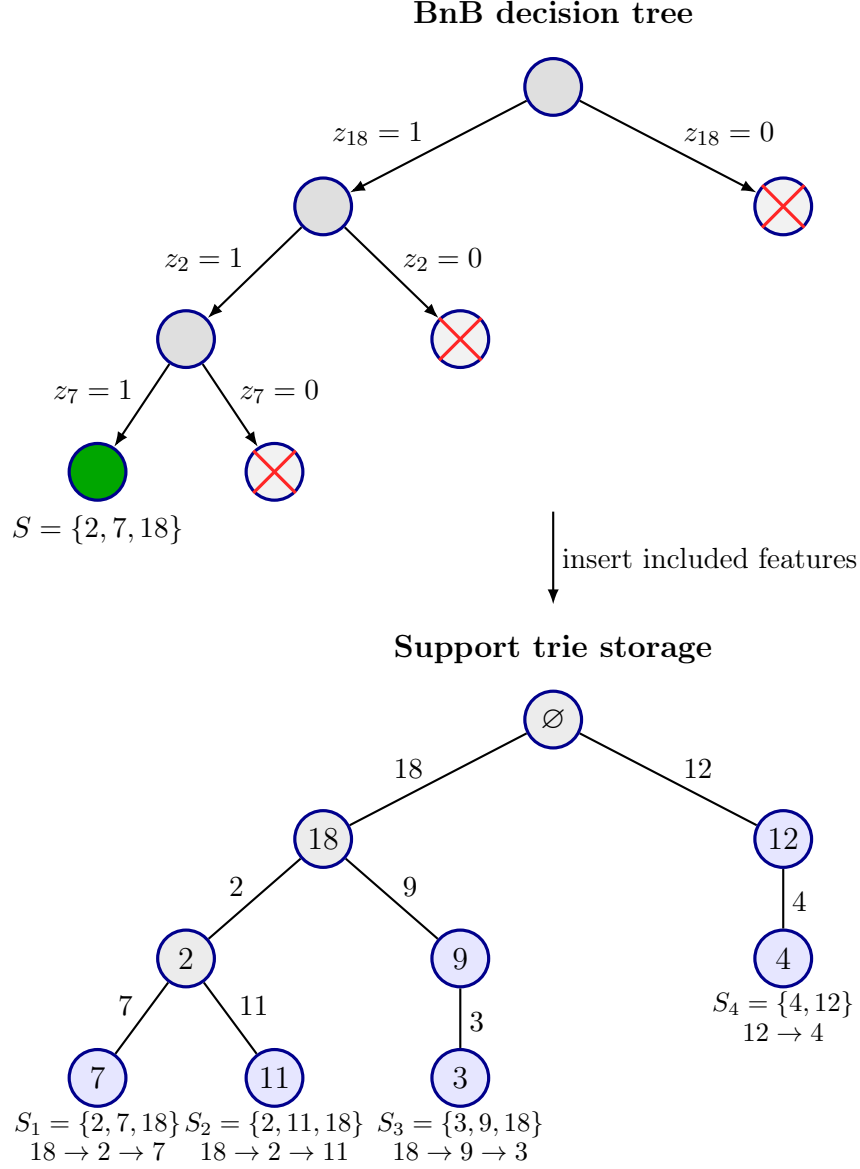
\begin{figure}[!ht]
\centering
\begin{tikzpicture}[
    scale=0.78,
    x=1.5cm,
    y=1.5cm,
    every node/.style={font=\large},
    bnode/.style={circle, draw=blue!55!black, very thick, fill=gray!25, minimum size=7.5mm, inner sep=0pt},
    leafbad/.style={circle, draw=blue!55!black, very thick, fill=gray!10, minimum size=7.5mm, inner sep=0pt},
    leafgood/.style={circle, draw=blue!55!black, very thick, fill=green!65!black, minimum size=7.5mm, inner sep=0pt},
    tnode/.style={circle, draw=blue!55!black, very thick, fill=gray!15, minimum size=7.5mm, inner sep=0pt},
    tleaf/.style={circle, draw=blue!55!black, very thick, fill=blue!10, minimum size=7.5mm, inner sep=0pt},
    edge/.style={-{Latex[length=2.0mm]}, thick},
    trieedge/.style={thick},
    elabel/.style={fill=white, inner sep=1pt, font=\normalsize},
    title/.style={font=\large\bfseries}
]
    \node[title] at (0,0.85) {BnB decision tree};
    \node[bnode] (b0) at (0,0) {};
    \node[bnode] (b1) at (-2.6,-1.35) {};
    \node[leafbad] (b2) at (2.6,-1.35) {};
    \node[bnode] (b3) at (-4.15,-2.85) {};
    \node[leafbad] (b4) at (-1.05,-2.85) {};
    \node[leafgood] (b5) at (-5.15,-4.35) {};
    \node[leafbad] (b6) at (-3.15,-4.35) {};

    \draw[edge] (b0) -- node[elabel, pos=0.55, above left=1pt] {$z_{18}=1$} (b1);
    \draw[edge] (b0) -- node[elabel, pos=0.55, above right=1pt] {$z_{18}=0$} (b2);
    \draw[edge] (b1) -- node[elabel, pos=0.55, above left=1pt] {$z_2=1$} (b3);
    \draw[edge] (b1) -- node[elabel, pos=0.55, above right=1pt] {$z_2=0$} (b4);
    \draw[edge] (b3) -- node[elabel, pos=0.55, above left=1pt] {$z_7=1$} (b5);
    \draw[edge] (b3) -- node[elabel, pos=0.55, above right=1pt] {$z_7=0$} (b6);

    \draw[red!85, very thick] (b2.south west) -- (b2.north east);
    \draw[red!85, very thick] (b2.north west) -- (b2.south east);
    \draw[red!85, very thick] (b4.south west) -- (b4.north east);
    \draw[red!85, very thick] (b4.north west) -- (b4.south east);
    \draw[red!85, very thick] (b6.south west) -- (b6.north east);
    \draw[red!85, very thick] (b6.north west) -- (b6.south east);

    \node[font=\normalsize, align=center] at (-5.15,-5.0) {$S=\{2,7,18\}$};
    \draw[edge] (0,-4.8) -- node[elabel, right=2pt] {insert included features} (0,-5.85);

    \node[title] at (0,-6.35) {Support trie storage};
    \node[tnode] (root) at (0,-7.15) {$\varnothing$};

    \node[tnode] (n18) at (-2.6,-8.5) {$18$};
    \node[tleaf] (n12) at (2.6,-8.5) {$12$};

    \node[tnode] (n182) at (-4.15,-9.85) {$2$};
    \node[tleaf] (n189) at (-1.05,-9.85) {$9$};
    \node[tleaf] (n124) at (2.6,-9.85) {$4$};

    \node[tleaf] (n1827) at (-5.15,-11.2) {$7$};
    \node[tleaf] (n18211) at (-3.15,-11.2) {$11$};
    \node[tleaf] (n1893) at (-1.05,-11.2) {$3$};

    \draw[trieedge] (root) -- node[elabel, pos=0.55, above left=1pt] {$18$} (n18);
    \draw[trieedge] (root) -- node[elabel, pos=0.55, above right=1pt] {$12$} (n12);
    \draw[trieedge] (n18) -- node[elabel, pos=0.55, above left=1pt] {$2$} (n182);
    \draw[trieedge] (n18) -- node[elabel, pos=0.55, above right=1pt] {$9$} (n189);
    \draw[trieedge] (n12) -- node[elabel, pos=0.55, right=2pt] {$4$} (n124);
    \draw[trieedge] (n182) -- node[elabel, pos=0.55, above left=1pt] {$7$} (n1827);
    \draw[trieedge] (n182) -- node[elabel, pos=0.55, above right=1pt] {$11$} (n18211);
    \draw[trieedge] (n189) -- node[elabel, pos=0.55, right=2pt] {$3$} (n1893);

    \node[align=center, font=\small] at (-5.15,-11.85)
    {$S_1=\{2,7,18\}$\\[-1pt]$18\to2\to7$};
    \node[align=center, font=\small] at (-3.15,-11.85)
    {$S_2=\{2,11,18\}$\\[-1pt]$18\to2\to11$};
    \node[align=center, font=\small] at (-1.05,-11.85)
    {$S_3=\{3,9,18\}$\\[-1pt]$18\to9\to3$};
    \node[align=center, font=\small] at (2.6,-10.5)
    {$S_4=\{4,12\}$\\[-1pt]$12\to4$};
\end{tikzpicture}
\caption{
    The BnB search tree and the support trie are different objects.
    BnB edges are binary branch decisions, such as $z_j=1$ or $z_j=0$.
    After a Rashomon support is found, the included feature sequence is inserted into the trie; excluded branch decisions are not stored.
    In this example, the BnB path $z_{18}=1$, $z_2=1$, $z_7=1$ is stored as the trie path $18\to2\to7$.
}
\label{fig:rashomon_support_trie}
\end{figure}

%% file: references.bib
@misc{liu2026gpufriendlylinearlyconvergentfirstorder,
      title={GPU-friendly and Linearly Convergent First-order Methods for Certifying Optimal $k$-sparse GLMs},
      author={Jiachang Liu and Andrea Lodi and Soroosh Shafiee},
      year={2026},
      eprint={2603.01306},
      archivePrefix={arXiv},
      primaryClass={math.OC},
      url={https://arxiv.org/abs/2603.01306},
}

@inproceedings{liu2025scalable,
  title={Scalable First-order Method for Certifying Optimal k-Sparse {GLM}s},
  author={Liu, Jiachang and Shafiee, Soroosh and Lodi, Andrea},
  booktitle={Proceedings of the 42nd International Conference on Machine Learning},
  pages={39455--39481},
  year={2025},
  editor={Singh, Aarti and Fazel, Maryam and Hsu, Daniel and Lacoste-Julien, Simon and Berkenkamp, Felix and Maharaj, Tegan and Wagstaff, Kiri and Zhu, Jerry},
  volume={267},
  series={Proceedings of Machine Learning Research},
  month={13--19 Jul},
  publisher={PMLR}
}

@inproceedings{raina2009large,
  title={Large-scale deep unsupervised learning using graphics processors},
  author={Raina, Rajat and Madhavan, Anand and Ng, Andrew Y.},
  booktitle={Proceedings of the 26th Annual International Conference on Machine Learning},
  pages={873--880},
  year={2009},
}

@inproceedings{krizhevsky2012imagenet,
  title={{ImageNet} classification with deep convolutional neural networks},
  author={Krizhevsky, Alex and Sutskever, Ilya and Hinton, Geoffrey E.},
  booktitle={Advances in Neural Information Processing Systems},
  volume={25},
  pages={1097--1105},
  year={2012},
}

@inproceedings{applegate2021practical,
  title={Practical large-scale linear programming using primal-dual hybrid gradient},
  author={Applegate, David and D{\'\i}az, Mateo and Hinder, Oliver and Lu, Haihao and Lubin, Miles and O'Donoghue, Brendan and Schudy, Warren},
  booktitle={Advances in Neural Information Processing Systems},
  pages={20243--20257},
  year={2021},
}

@misc{lu2023cupdlp,
  title={{cuPDLP-C}: A Strengthened Implementation of {cuPDLP} for Linear Programming by {C} language},
  author={Lu, Haihao and Yang, Jinwen and Hu, Haodong and Huangfu, Qi and Liu, Jinsong and Liu, Tianhao and Ye, Yinyu and Zhang, Chuwen and Ge, Dongdong},
  year={2024},
  eprint={2312.14832},
  archivePrefix={arXiv},
  primaryClass={math.OC},
  url={https://arxiv.org/abs/2312.14832},
}

@misc{lu2023practical,
  title={A Practical and Optimal First-Order Method for Large-Scale Convex Quadratic Programming},
  author={Lu, Haihao and Yang, Jinwen},
  year={2025},
  eprint={2311.07710},
  archivePrefix={arXiv},
  primaryClass={math.OC},
  url={https://arxiv.org/abs/2311.07710},
}

@misc{lin2025practicalgpu,
  title={A Practical {GPU}-Enhanced Matrix-Free Primal-Dual Method for Large-Scale Conic Programs},
  author={Lin, Zhenwei and Xiong, Zikai and Ge, Dongdong and Ye, Yinyu},
  year={2025},
  eprint={2505.00311},
  archivePrefix={arXiv},
  primaryClass={math.OC},
  url={https://arxiv.org/abs/2505.00311},
}

@article{ceria1999convex,
  title={Convex programming for disjunctive convex optimization},
  author={Ceria, Sebasti{\'a}n and Soares, Jo{\~a}o},
  journal={Mathematical Programming},
  volume={86},
  number={3},
  pages={595--614},
  year={1999},
}

@article{frangioni2006perspective,
  title={Perspective cuts for a class of convex 0--1 mixed integer programs},
  author={Frangioni, Antonio and Gentile, Claudio},
  journal={Mathematical Programming},
  volume={106},
  number={2},
  pages={225--236},
  year={2006},
}

@article{gunluk2010perspective,
  title={Perspective reformulations of mixed integer nonlinear programs with indicator variables},
  author={G{\"u}nl{\"u}k, Oktay and Linderoth, Jeff},
  journal={Mathematical Programming},
  volume={124},
  number={1--2},
  pages={183--205},
  year={2010},
}

@article{busing2022monotone,
  title={Monotone regression: A simple and fast {O} (n) {PAVA} implementation},
  author={Busing, Frank MTA},
  journal={Journal of Statistical Software},
  volume={102},
  number={Code Snippet 1},
  pages={1--25},
  year={2022},
}

@book{rockafellar1970convex,
  title={Convex Analysis},
  author={Rockafellar, R. Tyrrell},
  year={1970},
  publisher={Princeton University Press},
}

@book{beck2017first,
  title={First-Order Methods in Optimization},
  author={Beck, Amir},
  year={2017},
  publisher={SIAM},
}

@article{achterberg2005branching,
  title={Branching rules revisited},
  author={Achterberg, Tobias and Koch, Thorsten and Martin, Alexander},
  journal={Operations Research Letters},
  volume={33},
  number={1},
  pages={42--54},
  year={2005},
  doi={10.1016/j.orl.2004.04.002},
}

@article{achterberg2009scip,
  title={{SCIP}: solving constraint integer programs},
  author={Achterberg, Tobias},
  journal={Mathematical Programming Computation},
  volume={1},
  number={1},
  pages={1--41},
  year={2009},
  doi={10.1007/s12532-008-0001-1},
}

@article{linderoth1999computational,
  title={A computational study of search strategies for mixed integer programming},
  author={Linderoth, J. T. and Savelsbergh, M. W. P.},
  journal={INFORMS Journal on Computing},
  volume={11},
  number={2},
  pages={173--187},
  year={1999},
  doi={10.1287/ijoc.11.2.173},
}

@article{morrison2016branch,
  title={Branch-and-bound algorithms: A survey of recent advances in searching, branching, and pruning},
  author={Morrison, David R. and Jacobson, Sheldon H. and Sauppe, Jason J. and Sewell, Edward C.},
  journal={Discrete Optimization},
  volume={19},
  pages={79--102},
  year={2016},
  doi={10.1016/j.disopt.2016.01.005},
}

@article{breiman2001statistical,
  title={Statistical Modeling: The Two Cultures (with Comments and a Rejoinder by the Author)},
  author={Breiman, Leo},
  journal={Statistical Science},
  volume={16},
  number={3},
  pages={199--231},
  year={2001}
}

@article{ustun2019learning,
  title={Learning Optimized Risk Scores},
  author={Ustun, Berk and Rudin, Cynthia},
  journal={Journal of Machine Learning Research},
  volume={20},
  number={150},
  pages={1--75},
  year={2019}
}

@article{ustun2016supersparse,
  title={Supersparse Linear Integer Models for Optimized Medical Scoring Systems},
  author={Ustun, Berk and Rudin, Cynthia},
  journal={Machine Learning},
  volume={102},
  number={3},
  pages={349--391},
  year={2016},
  publisher={Springer}
}

@article{fisher2019all,
  title={All Models are Wrong, but Many are Useful: Learning a Variable's Importance by Studying an Entire Class of Prediction Models Simultaneously},
  author={Fisher, Aaron and Rudin, Cynthia and Dominici, Francesca},
  journal={Journal of Machine Learning Research},
  volume={20},
  number={177},
  pages={1--81},
  year={2019}
}

@article{dong2020exploring,
  title={Exploring the Cloud of Variable Importance for the Set of All Good Models},
  author={Dong, Jiayun and Rudin, Cynthia},
  journal={Nature Machine Intelligence},
  volume={2},
  pages={810--824},
  year={2020}
}

@inproceedings{liu2022fastsparse,
  title={Fast Sparse Classification for Generalized Linear and Additive Models},
  author={Liu, Jiachang and Zhong, Chudi and Seltzer, Margo and Rudin, Cynthia},
  booktitle={Proceedings of the 25th International Conference on Artificial Intelligence and Statistics},
  pages={9304--9333},
  year={2022}
}

@inproceedings{xin2022exploring,
  title={Exploring the Whole {Rashomon} Set of Sparse Decision Trees},
  author={Xin, Rui and Zhong, Chudi and Chen, Zhi and Takagi, Takuya and Seltzer, Margo and Rudin, Cynthia},
  booktitle={Advances in Neural Information Processing Systems},
  volume={35},
  year={2022}
}

@inproceedings{liu2022fasterrisk,
  title={{FasterRisk}: Fast and Accurate Interpretable Risk Scores},
  author={Liu, Jiachang and Zhong, Chudi and Li, Boxuan and Seltzer, Margo and Rudin, Cynthia},
  booktitle={Advances in Neural Information Processing Systems},
  volume={35},
  year={2022}
}

@inproceedings{donnelly2023rashomon,
  title={The {Rashomon} Importance Distribution: Getting {RID} of Unstable, Single Model-Based Variable Importance},
  author={Donnelly, Jon and Katta, Srikar and Rudin, Cynthia and Browne, Edward P.},
  booktitle={Advances in Neural Information Processing Systems},
  volume={36},
  year={2023}
}

@inproceedings{zhong2023sparsegam,
  title={Exploring and Interacting with the Set of Good Sparse Generalized Additive Models},
  author={Zhong, Chudi and Chen, Zhi and Liu, Jiachang and Seltzer, Margo and Rudin, Cynthia},
  booktitle={Advances in Neural Information Processing Systems},
  volume={36},
  pages={56673--56699},
  year={2023}
}

@inproceedings{donnelly2025rashomonproto,
  title={{Rashomon} Sets for Prototypical-Part Networks: Editing Interpretable Models in Real-Time},
  author={Donnelly, Jon and Guo, Zhicheng and Barnett, Alina Jade and McTavish, Hayden and Chen, Chaofan and Rudin, Cynthia},
  booktitle={Proceedings of the IEEE/CVF Conference on Computer Vision and Pattern Recognition},
  year={2025}
}

@inproceedings{babbar2025nearoptimal,
  title={Near-Optimal Decision Trees in a {SPLIT} Second},
  author={Babbar, Varun and McTavish, Hayden and Rudin, Cynthia and Seltzer, Margo},
  booktitle={International Conference on Machine Learning},
  year={2025}
}

@inproceedings{mctavish2025predictiveequivalence,
  title={Leveraging Predictive Equivalence in Decision Trees},
  author={McTavish, Hayden and Boner, Zachery and Donnelly, Jon and Seltzer, Margo and Rudin, Cynthia},
  booktitle={International Conference on Machine Learning},
  year={2025}
}

@inproceedings{rudin2024amazing,
  title={Position: Amazing Things Come From Having Many Good Models},
  author={Rudin, Cynthia and Zhong, Chudi and Semenova, Lesia and Seltzer, Margo and Parr, Ronald and Liu, Jiachang and Katta, Srikar and Donnelly, Jon and Chen, Harry and Boner, Zachery},
  booktitle={Proceedings of the 41st International Conference on Machine Learning},
  pages={42783--42795},
  year={2024}
}

@misc{nguyen2026realitrees,
  title={{REALITrees}: {Rashomon} Ensemble Active Learning for Interpretable Trees},
  author={Nguyen, Simon D. and McTavish, Hayden and Hoffman, Kentaro and Rudin, Cynthia and McCormick, Tyler H.},
  year={2026},
  eprint={2603.22750},
  archivePrefix={arXiv},
  primaryClass={cs.LG},
  url={https://arxiv.org/abs/2603.22750}
}

@misc{haghighat2026resolving,
  title={Resolving Predictive Multiplicity for the {Rashomon} Set},
  author={Haghighat, Parian and Anahideh, Hadis and Rudin, Cynthia},
  year={2026},
  eprint={2601.09071},
  archivePrefix={arXiv},
  primaryClass={cs.LG},
  url={https://arxiv.org/abs/2601.09071}
}

@inproceedings{atamturk2020safe,
  title={Safe Screening Rules for {$\ell_0$}-Regression from Perspective Relaxations},
  author={Atamturk, Alper and G{\'o}mez, Andr{\'e}s},
  booktitle={Proceedings of the 37th International Conference on Machine Learning},
  pages={421--430},
  year={2020}
}

@article{xie2020scalable,
  title={Scalable Algorithms for the Sparse Ridge Regression},
  author={Xie, Weijun and Deng, Xinwei},
  journal={SIAM Journal on Optimization},
  volume={30},
  number={4},
  pages={3359--3386},
  year={2020}
}

@article{bertsimas2020sparse1,
  title={Sparse Regression: Scalable Algorithms and Empirical Performance},
  author={Bertsimas, Dimitris and Pauphilet, Jean and Van Parys, Bart},
  journal={Statistical Science},
  volume={35},
  number={4},
  pages={555--578},
  year={2020}
}

@article{bertsimas2020sparse2,
  title={Sparse High-Dimensional Regression: Exact Scalable Algorithms and Phase Transitions},
  author={Bertsimas, Dimitris and Van Parys, Bart},
  journal={The Annals of Statistics},
  volume={48},
  number={1},
  pages={300--323},
  year={2020}
}

@article{hazimeh2020fast,
  title={Fast Best Subset Selection: Coordinate Descent and Local Combinatorial Optimization Algorithms},
  author={Hazimeh, Hussein and Mazumder, Rahul},
  journal={Operations Research},
  volume={68},
  number={5},
  pages={1517--1537},
  year={2020}
}

@article{hazimeh2022sparse,
  title={Sparse Regression at Scale: Branch-and-Bound Rooted in First-Order Optimization},
  author={Hazimeh, Hussein and Mazumder, Rahul and Saab, Ali},
  journal={Mathematical Programming},
  volume={196},
  number={1},
  pages={347--388},
  year={2022}
}

@inproceedings{guyard2024el0ps,
  title={A New Branch-and-Bound Pruning Framework for {$\ell_0$}-Regularized Problems},
  author={Guyard, Th{\'e}o and Herzet, C{\'e}dric and Elvira, Cl{\'e}ment and Arslan, Ayse-Nur},
  booktitle={Proceedings of the 41st International Conference on Machine Learning},
  pages={48077--48096},
  year={2024}
}

@inproceedings{liu2024okridge,
  title={{OKR}idge: Scalable Optimal {$k$}-Sparse Ridge Regression},
  author={Liu, Jiachang and Rosen, Sam and Zhong, Chudi and Rudin, Cynthia},
  booktitle={Advances in Neural Information Processing Systems},
  pages={41076--41258},
  year={2023}
}

@article{bertsimas2023learning,
  title={Learning Sparse Nonlinear Dynamics via Mixed-Integer Optimization},
  author={Bertsimas, Dimitris and Gurnee, Wes},
  journal={Nonlinear Dynamics},
  volume={111},
  number={7},
  pages={6585--6604},
  year={2023}
}

@article{dedieu2021learning,
  title={Learning Sparse Classifiers: Continuous and Mixed Integer Optimization Perspectives},
  author={Dedieu, Antoine and Hazimeh, Hussein and Mazumder, Rahul},
  journal={Journal of Machine Learning Research},
  volume={22},
  number={135},
  pages={1--47},
  year={2021}
}

@misc{han2024accelerating,
  title={Accelerating Low-Rank Factorization-Based Semidefinite Programming Algorithms on {GPU}},
  author={Han, Qiushi and Lin, Zhenwei and Liu, Hanwen and Chen, Caihua and Deng, Qi and Ge, Dongdong and Ye, Yinyu},
  year={2024},
  eprint={2407.15049},
  archivePrefix={arXiv},
  primaryClass={math.OC},
  url={https://arxiv.org/abs/2407.15049}
}

@misc{de2024power,
  title={On the Power of Linear Programming for {K}-Means Clustering},
  author={De Rosa, Antonio and Khajavirad, Aida and Wang, Yakun},
  year={2024},
  eprint={2402.01061},
  archivePrefix={arXiv},
  primaryClass={math.OC},
  url={https://arxiv.org/abs/2402.01061}
}

@misc{blin2026batchedlp,
  title={Batched First-Order Methods for Parallel {LP} Solving in {MIP}},
  author={Blin, Nicolas and Gualandi, Stefano and Maes, Christopher and Lodi, Andrea and Stellato, Bartolomeo},
  year={2026},
  eprint={2601.21990},
  archivePrefix={arXiv},
  primaryClass={math.OC},
  url={https://arxiv.org/abs/2601.21990}
}

@misc{meng2026gpu,
  title={A {GPU}-Accelerated Nonlinear Branch-and-Bound Framework for Sparse Linear Models},
  author={Meng, Xiang and Lucas, Ryan and Mazumder, Rahul},
  year={2026},
  eprint={2602.04551},
  archivePrefix={arXiv},
  primaryClass={math.OC},
  url={https://arxiv.org/abs/2602.04551}
}

@misc{asuncion2007uci,
  title={{The UCI Machine Learning Repository}},
  author={Asuncion, Arthur and Newman, David},
  year={2007},
  publisher={Irvine, CA, USA}
}

@article{atamturk2020supermodularity,
  title={Supermodularity and valid inequalities for quadratic optimization with indicators},
  author={Atamt{\"u}rk, Alper and G{\'o}mez, Andr{\'e}s},
  journal={Mathematical Programming},
  volume={201},
  number={1--2},
  pages={295--338},
  year={2023},
  publisher={Springer}
}

@article{atamturk2021sparse,
  title={Sparse and Smooth Signal Estimation: Convexification of {$\ell_0$}-Formulations},
  author={Atamt{\"u}rk, Alper and G{\'o}mez, Andr{\'e}s and Han, Shaoning},
  journal={Journal of Machine Learning Research},
  volume={22},
  number={52},
  pages={1--43},
  year={2021}
}

@misc{gurobi,
  author={{Gurobi Optimization, LLC}},
  title={{Gurobi Optimizer Reference Manual}},
  year={2025}
}

@manual{mosek,
  author={{MOSEK ApS}},
  title={The {MOSEK} optimization toolbox for {MATLAB} manual. Version 11.0.4},
  year={2025}
}

@misc{santander,
  author={Piedra, Mercedes and Dane, Sohier and Jimenez, Soraya},
  title={Santander Customer Transaction Prediction},
  year={2019},
  howpublished={\url{https://kaggle.com/competitions/santander-customer-transaction-prediction}},
  note={{Kaggle} competition}
}

@article{shafiee2024constrained,
  title={Constrained optimization of rank-one functions with indicator variables},
  author={Shafiee, Soroosh and K{\i}l{\i}n{\c{c}}-Karzan, Fatma},
  journal={Mathematical Programming},
  volume={208},
  number={1--2},
  pages={533--579},
  year={2024},
  publisher={Springer}
}

@inproceedings{wei2020convexification,
  title={On the convexification of constrained quadratic optimization problems with indicator variables},
  author={Wei, Linchuan and G{\'o}mez, Andr{\'e}s and K{\"u}{\c{c}}{\"u}kyavuz, Simge},
  booktitle={Integer Programming and Combinatorial Optimization},
  pages={433--447},
  year={2020}
}

@article{wei2022ideal,
  title={Ideal formulations for constrained convex optimization problems with indicator variables},
  author={Wei, Linchuan and G{\'o}mez, Andr{\'e}s and K{\"u}{\c{c}}{\"u}kyavuz, Simge},
  journal={Mathematical Programming},
  volume={192},
  number={1},
  pages={57--88},
  year={2022},
  publisher={Springer}
}

@misc{corduk2025gpuacceleratedprimalheuristics,
  title={{GPU}-Accelerated Primal Heuristics for Mixed Integer Programming},
  author={Akif {\c{C}}{\"o}rd{\"u}k and Piotr Sielski and Alice Boucher and Kumar Aatish},
  year={2025},
  eprint={2510.20499},
  archivePrefix={arXiv},
  primaryClass={math.OC},
  url={https://arxiv.org/abs/2510.20499}
}
